\relax
\documentclass[letterpaper]{article} % DO NOT CHANGE THIS
\usepackage{aaai20}  % DO NOT CHANGE THIS
\usepackage{times}  % DO NOT CHANGE THIS
\usepackage{helvet} % DO NOT CHANGE THIS
\usepackage{courier}  % DO NOT CHANGE THIS
\usepackage[hyphens]{url}  % DO NOT CHANGE THIS
\usepackage{graphicx} % DO NOT CHANGE THIS
\usepackage{graphicx}  % DO NOT CHANGE THIS

\usepackage{algorithm}
\usepackage{algorithmicx}
\usepackage{algpseudocode}
\usepackage{amsfonts}
\usepackage{amsmath}
\usepackage{amssymb}
\usepackage{mathtools}
\usepackage{amsthm}
\usepackage{booktabs}
\usepackage{subfigure}
\usepackage{float}
\usepackage{color}
\newtheorem{definition}{Definition}
\newtheorem{theorem}{Theorem}
\newtheorem{remark}{Remark}
\newtheorem{corollary}{Corollary}
\newtheorem{assumption}{Assumption}
\newtheorem{lemma}{Lemma}

\title{Stability and Generalization of Differentially Private Minimax Problems}
\author{\Large \textbf{Yilin Kang, Yong Liu, Jian Li, Weiping Wang}\\ % All authors must be in the same font size and format. Use \Large and \textbf to achieve this result when breaking a line
}

\begin{document}

\maketitle

\begin{abstract}
In the field of machine learning, many problems can be formulated as the minimax problem, including reinforcement learning, generative adversarial networks, to just name a few.
So the minimax problem has attracted a huge amount of attentions from researchers in recent decades.
However, there is relatively little work on studying the privacy of the general minimax paradigm.
In this paper, we focus on the privacy of the general minimax setting, combining differential privacy together with minimax optimization paradigm.
Besides, via algorithmic stability theory, we theoretically analyze the high probability generalization performance of the differentially private minimax algorithm under the strongly-convex-strongly-concave condition.
To the best of our knowledge, this is the first time to analyze the generalization performance of general minimax paradigm, taking differential privacy into account.
\end{abstract}

\section{1. Introduction}
In the field of machine learning, many problems can be formulated as the minimax problem, including adversarial learning \cite{goodfellow2014generative}, reinforcement learning \cite{du2017stochastic,dai2018sbeed}, AUC maximization \cite{zhao2011online,gao2013one,ying2016stochastic,liu2018fast,lei2021stochastic}, robust optimization \cite{chen2017robust,namkoong2017variance}, and distributed computing \cite{shamma2008cooperative,mateos2010distributed,razaviyayn2020nonconvex}, to mention but a few.
In the minimax problem, there are two groups of decision variables, one for minimization and the other for maximization \cite{lei2021stability}.
To solve the minimax problem, various optimization algorithms have been designed, such as Gradient Descent Ascent (GDA), Stochastic Gradient Descent Ascent (SGDA), Alternating Gradient Descent Ascent
(AGDA), Proximal Point Method (PPM), etc \cite{farnia2021train,lei2021stability}.

Like other machine learning problems, the minimax problem also faces the privacy troubles, because tremendous individual's data have to be collected for training \cite{phan2020scalable,wang2021differentially}.
In the real scenarios, not only the original data leakages the sensitive information, the machine learning model also causes privacy issues \cite{fredrikson2014privacy,shokri2017membership}.
Under these circumstances, \cite{dwork2006calibrating} proposes a theoretically rigorous tool: Differential Privacy (DP) \cite{dwork2014the}, to protect the sensitive information of individuals who participate in the training dataset, by introducing random noise to the model.
There are mainly three approaches to guarantee DP: output perturbation, objective perturbation, and gradient perturbation \cite{chaudhuri2011differentially,song2013stochastic}.
There exist some works to combine differential privacy with some of the minimax problems, such as DP-adversarial learning \cite{xu2019ganobfuscator,phan2020scalable,giraldo2020adversarial,lin2021on}, DP-reinforcement learning \cite{vietri2020private,chen2021an}, DP-AUC maximization \cite{huai2020pairwise,wang2021differentially,yang2021stability}, and DP-robust optimization \cite{lecuyer2019certified}.
However, all the researches mentioned above only focus on particular models (e.g. Generative Adversarial Networks (GANs)), but not on the general minimax paradigm.
Besides, there is no analysis on the generalization error of the general DP-minimax problem to the best of our knowledge.

To solve these problems, in this paper, we concentrate on differentially private minimax optimization, provide privacy guarantees and analyze the generalization performance of the DP minimax model (we pay more attentions on the generalization part).
Due to the simplicity, GDA is one of the most widespread usded optimization methods in the field of minimax problem, so we focus on GDA in this paper.
Besides, considering that the high probability generalization performance of machine learning models is paid more attentions, we use algorithmic stability theory, in particular, argument stability, to get several different generalization measures of DP-GDA model, under high probability condition.
The contributions of this paper include:
1. In the minimax problem, there are two decision variables: one for minimization and one for maximization, so there are various generalization measures due to different coupling combinations \cite{farnia2021train,lei2021stability,zhang2021generalization}.
In this paper, we analyze almost all existing generalization measures for our proposed DP-GDA algorithm, via algorithm stability theory and provide corresponding high probability bounds.
And our results are better than previous results analyzed for particular DP-minimax model (such as DP-AUC maximization).
To the best of our knowledge, this is the first time to give generalization bounds for the general minimax paradigm.
2. In the minimax problem, the minimization parameter $\bold{w}$ differs when it comes to different maximization parameters $\bold{v}$ (and vice versa), which brings challenges to the theoretical analysis.
Moreover, for differentially private models, random noise is an essential part and it makes the results worse.
In this paper, to overcome these problems, novel decomposition methods are applied and sharper generalization bounds are achieved.

The rest of the paper is organized as follows.
The related work is discussed in Section 2.
Preliminaries are introduced in Section 3.
The algorithm DP-GDA along with the privacy guarantees are given in Section 4.
We analyze the generalization performance of DP-GDA in Section 5.
We compare our results with existed works in Section 6.
Finally, we conclude the paper in Section 7.
And all the proofs are given in the Appendix.

\section{2. Related Work}

For the minimax problem, there is a long list of works discussing the convergence analysis and the empirical risk under convex-concave condition \cite{mokhtari2020a,yan2020optimal,lin2020near,wang2020improved,yoon2021accelerated}, nonconvex-concave condition \cite{luo2020stochastic,lu2020hybird,lin2020on,chen2021proximal}, and nonconvex-nonconcave condition \cite{loizou2020stochastic,yang2020global,wang2020on,liu2021first,diakonikolas2021efficient,loizou2021stochastic,fiez2021local}.
However, there is relatively little work concerning the generalization performance.
\cite{zhang2021generalization} analyzes the generalization performance of the empirical saddle point (ESP) solution in the minimax problem; \cite{farnia2021train} discusses the generalization performance of several popular optimization algorithms in the minimax problem: GDA, GDmax, SGDA, SGDmax, and PPM.
And \cite{zhang2021generalization} and \cite{farnia2021train} only focus on the expectation generalization bounds.
Besides, \cite{lei2021stability} gives high probability generalization bounds, of the order $\mathcal{O}(1/\sqrt{n})$, where $n$ is the number of training data.

To analyze the generalization performance, complexity theory \cite{bartlett2002localized} and algorithmic stability \cite{bousquet2002stability} are popular tools.
On one hand, some existed works analyze the generalization performance of differentially private models via complexity theory.
For example, \cite{chaudhuri2011differentially} gives high probability excess population risk bound under strongly convex conditions and \cite{kifer2012private} gives the excess population risk bound in expectation.
On the other hand, algorithmic stability is a fundamental concept in learning theory, and it captures the fluctuations on the model caused by modifying one of the data instances.
Algorithmic stability has been widely studied in recent decades, including uniform stability \cite{bousquet2002stability,hardt2016train}, hypothesis stability \cite{bousquet2002stability}, argument stability \cite{liu2017algorithmic,bassily2020stability}, locally elastic stability \cite{deng2021toward}.
And there is a long list of works analyzing the high probability generalization bounds for differentially private models via algorithmic stability \cite{wu2017bolt,bassily2019private,feldman2020private,bassily2020stability,wang2021differentiallyarxiv}.
Besides, \cite{lei2021stability,farnia2021train,zhang2021generalization} have extended the uniform stability and argument stability to the minimax setting, and a new stability concept: weakly stability is designed for the minimax problem in \cite{lei2021stability}.
However, the analysis on the generalization performance of DP-minimax problem is still a blank.

Meanwhile, there are some works combining differential privacy with some particular minimax problems.
\cite{xu2019ganobfuscator} designs GANobfuscator, which guarantees DP of GAN via gradient perturbation method.
Besides, \cite{giraldo2020adversarial} discusses the inherent privacy of GANs, from the view of DP.
However, these works do not give any theoretical utility guarantees.
\cite{wang2021differentially} proposes output and objective perturbation methods to guarantee the differential privacy of AUC maximization, and achieves $\mathcal{O}(1/\sqrt{n})$ high probability excess population risk bound.
Considering that AUC maximization can be seemed as one of the pairwise learning problems, \cite{huai2020pairwise,yang2021stability} analyze the privacy guarantees and the generalization bounds of DP-AUC maximization from the pairwise perspective, they also achieve excess population risk bound of the order $\mathcal{O}(1/\sqrt{n})$ with high probability.
For reinforcement learning, \cite{chen2021an} proposes a DP version on the platform Vehicular ad hoc network (VANET), and \cite{vietri2020private} not only discusses the privacy, but also gives the probably approximately correct (PAC) and regret bounds.
However, to the best of our knowledge, there is no existing work concentrating on DP of the general minimax problem.

\section{3. Preliminaries}

\subsection{3.1 Differential Privacy}
Let $\mathcal{P}$ be a probability measure defined on the data space $\mathcal{Z}$ and let dataset $S=\{z_1,\cdots,z_n\}$ be independent drawn from $\mathcal{P}$.
Datasets $S,S'$ differing by at most one data instance are denoted by $S\sim S'$, called adjacent datasets.
\begin{definition}\label{def4}[Differential Privacy \cite{dwork2014the}]
	Algorithm $\mathcal{A}:\mathcal{Z}^n\rightarrow\mathbb{R}^p$ is ($\epsilon,\delta$)-differential privacy (DP) if for all $S\sim S'$ and events $O\in range(\mathcal{A})$
	\begin{equation*}
	\mathbb{P}\left[\mathcal{A}(S)\in O\right]\leq e^{\epsilon}\mathbb{P}\left[\mathcal{A}(S')\in O\right]+\delta.
	\end{equation*}
\end{definition}
Differential privacy requires essentially the same distributions to be drawn over any adjacent datasets, so that the adversaries cannot infer whether an individual participates to the training process.
Some kind of attacks, such as attribute inference attack, membership inference attack, and memorization attack, can be thwarted by DP \cite{backes2016membership,jayaraman2019evaluating,carlini2019the}.

\subsection{3.2. Minimax Paradigm}
Denote two parameter spaces as $\mathcal{W},\mathcal{V}\in\mathbb{R}^p$, where $p$ is the parameter dimensions, then for the minimax problem, we define $\ell:\mathcal{W}\times\mathcal{V}\times\mathcal{Z}\rightarrow\mathbb{R}$ and consider
\begin{equation*}
\min_{\bold{w}\in\mathcal{W}}\max_{\bold{v}\in\mathcal{V}}L(\bold{w},\bold{v})\coloneqq\mathbb{E}_{z\sim\mathcal{P}}\left[\ell(\bold{w},\bold{v};z)\right].
\end{equation*}
Since the underlying $\mathcal{P}$ is always unknown, so empirical risk is designed as
\begin{equation*}
L_S(\bold{w},\bold{v})=\frac{1}{n}\sum_{i=1}^{n}\ell(\bold{w},\bold{v};z_i).
\end{equation*}

Denoting the model derived from dataset $S$ by applying algorithm $\mathcal{A}$ as $\mathcal{A}(S)=(\mathcal{A}_\bold{w}(S),\mathcal{A}_\bold{v}(S))\in\mathcal{W}\times\mathcal{V}$, in this paper, we focus on analyzing how well $\mathcal{A}(S)$ performs on the underlying distribution $\mathcal{P}$, i.e. the generalization performance.
There are several measures to demonstrate the generalization performance of the minimax model.

\begin{definition}[\cite{lei2021stability}]\label{def1} There are four generalization measures in the minimax problem.
	
	(a) \textbf{Weak Primal-Dual (PD) Risk}: The weak PD population risk of $(\bold{w},\bold{v})$ is defined as
	\begin{equation*}
	\triangle^{w}(\bold{w},\bold{v})=\sup_{\bold{v'}\in\mathcal{V}}\mathbb{E}\left[L(\bold{w},\bold{v'})\right]-\inf_{\bold{w'}\in\mathcal{W}}\mathbb{E}\left[L(\bold{w'},\bold{v})\right].
	\end{equation*}
	Corresponding empirical risk is defined as
	\begin{equation*}
	\triangle_S^{w}(\bold{w},\bold{v})=\sup_{\bold{v'}\in\mathcal{V}}\mathbb{E}\left[L_S(\bold{w},\bold{v'})\right]-\inf_{\bold{w'}\in\mathcal{W}}\mathbb{E}\left[L_S(\bold{w'},\bold{v})\right].
	\end{equation*}
	$\triangle^{w}(\bold{w},\bold{v})-\triangle_S^{w}(\bold{w},\bold{v})$ is referred to the weak PD generalization error of $(\bold{w},\bold{v})$.
	
	(b) \textbf{Strong Primal-Dual (PD) Risk}: The strong PD population risk of $(\bold{w},\bold{v})$ is defined as
	\begin{equation*}
	\triangle^{s}(\bold{w},\bold{v})=\sup_{\bold{v'}\in\mathcal{V}}L(\bold{w},\bold{v'})-\inf_{\bold{w'}\in\mathcal{W}}L(\bold{w'},\bold{v}).
	\end{equation*}
	Corresponding empirical risk is defined as
	\begin{equation*}
	\triangle_S^{s}(\bold{w},\bold{v})=\sup_{\bold{v'}\in\mathcal{V}}L_S(\bold{w},\bold{v'})-\inf_{\bold{w'}\in\mathcal{W}}L_S(\bold{w'},\bold{v}).
	\end{equation*}
	$\triangle^{s}(\bold{w},\bold{v})-\triangle_S^{s}(\bold{w},\bold{v})$ is referred to the strong PD generalization error of the model $(\bold{w},\bold{v})$.
	
	(c) \textbf{Primal Risk}: The primal population risk and empirical risk are w.r.t model $\bold{w}$, defined as
	\begin{equation*}
	R(\bold{w})=\sup_{\bold{v}\in\mathcal{V}}L(\bold{w},\bold{v}),\quad
	R_S(\bold{w})=\sup_{\bold{v}\in\mathcal{V}}L_S(\bold{w},\bold{v}).
	\end{equation*}
	If $R(\bold{w})$ is bounded by $R_S(\bold{w})$, we call this error the primal generalization error; if $R(\bold{w})$ is bounded by $\inf_{\bold{w'}\in\mathcal{W}}R(\bold{w'})$, we call this error the excess primal population risk.
	
	(d) \textbf{Plain Risk}: For model $(\bold{w},\bold{v})$, if $L(\bold{w},\bold{v})$ is bounded by $L_S(\bold{w},\bold{v})$, this error is called the plain generalization error.
\end{definition}

\begin{remark}\label{rem1}
	In Definition \ref{def1}, the primal risk is w.r.t only one of the parameters $\bold{w}$ and the plain risk are similar to which in traditional learning tasks.
	Besides, for weak and strong PD risks, one can easily get $\triangle^{w}(\bold{w},\bold{v})\leq\mathbb{E}\left[\triangle^{s}(\bold{w},\bold{v})\right]$ and $\triangle_S^{w}(\bold{w},\bold{v})\leq\mathbb{E}\left[\triangle_S^{s}(\bold{w},\bold{v})\right]$, so the weak PD risk is naturally bounded by the strongly PD risk bound.
\end{remark}

For a $p$-dimension vector $\bold{x}$, define its $\ell_2$ norm (Euclidean norm) as $\|\bold{x}\|_2=\big(\sum_{i=1}^{p}|\bold{x}_i|^2\big)^{1/2}$, and let $\langle\cdot,\cdot\rangle$ be the inner product.
A differentiable function $\ell:\mathcal{W}\rightarrow\mathbb{R}$ is called $\rho$-strongly-convex over $\bold{w}$ if for any $\bold{w},\bold{w}'$
\begin{equation*}
\ell(\bold{w})-\ell(\bold{w}')\geq\langle\nabla \ell(\bold{w}'),\bold{w}-\bold{w}'\rangle+\frac{\rho}{2}\|\bold{w}-\bold{w}'\|_2^2.
\end{equation*}
If $-\ell$ is $\rho$-strongly convex, then $\ell$ is $\rho$-strongly concave.

\begin{assumption}\label{a1}
	For the minimax problem, we say $\ell$ is $\rho$-strongly-convex-strongly-concave ($\rho$-SC-SC) if $\ell(\cdot,\bold{v})$ is $\rho$-strongly-convex for all $\bold{v}$ and $\ell(\bold{w},\cdot)$ is $\rho$-strongly-concave for all $\bold{w}$.
	In this paper, we focus on the $\rho$-SC-SC problem.
\end{assumption}

\begin{assumption}\label{a2}
	For $G>0$, $\ell(\bold{w},\bold{v};z)$ is $G$-Lipschitz if for any $\bold{w},\bold{v}$ and $z$
	\begin{equation*}
	\|\nabla_\bold{w}\ell(\bold{w},\bold{v};z)\|_2\leq G,\quad\|\nabla_\bold{v}\ell(\bold{w},\bold{v};z)\|_2\leq G.
	\end{equation*}
\end{assumption}

\begin{assumption}\label{a3}
	For $L>0$, $\ell(\bold{w},\bold{v};z)$ is $L$-smooth if for any $\bold{w},\bold{w}',\bold{v},\bold{v}'$ and $z$
	\begin{equation*}
	\left\{
	\begin{array}{l}
	\|\nabla_\bold{w}\ell(\bold{w},\bold{v};z)-\nabla_\bold{w}\ell(\bold{w'},\bold{v};z)\|_2\leq L\|\bold{w}-\bold{w}'\|_2, \\
	\|\nabla_\bold{w}\ell(\bold{w},\bold{v};z)-\nabla_\bold{w}\ell(\bold{w},\bold{v'};z)\|_2\leq L\|\bold{v}-\bold{v}'\|_2, \\
	\|\nabla_\bold{v}\ell(\bold{w},\bold{v};z)-\nabla_\bold{v}\ell(\bold{w'},\bold{v};z)\|_2\leq L\|\bold{w}-\bold{w}'\|_2, \\
	\|\nabla_\bold{v}\ell(\bold{w},\bold{v};z)-\nabla_\bold{v}\ell(\bold{w},\bold{v'};z)\|_2\leq L\|\bold{v}-\bold{v}'\|_2.
	\end{array}
	\right.
	\end{equation*}
\end{assumption}

\subsection{3.3. Algorithmic Stability}
Algorithmic stability is a popular tool to analyze the generalization performance of the machine learning model, which captures the difference between models derived from adjacent training datasets.
Some of the stabilities have been extended to the minimax problem, such as weakly stability, uniform stability and argument stability, in expectation or high probability \cite{zhang2021generalization,lei2021stability,farnia2021train}.

In this paper, we use argument stability to get the high probability generalization bounds, defined as follows.
\begin{definition}\label{def2}
	Algorithm $\mathcal{A}$ is $\gamma$-argument-stable ($\gamma>0$) if for any adjacent datasets $S\sim S'\in\mathcal{Z}^n$ we have
	\begin{equation*}
	\|\mathcal{A}_\bold{w}(S)-\mathcal{A}_\bold{w}(S')\|_2+\|\mathcal{A}_\bold{v}(S)-\mathcal{A}_\bold{v}(S')\|_2\leq\gamma.
	\end{equation*}
\end{definition}
In the minimax problem, argument stability demonstrates the gap between $\mathcal{A}(S)$ and $\mathcal{A}(S')$, formulated by the summation over $\bold{w}$ and $\bold{v}$.
Via property $G$-Lipschitz (Assumption \ref{a2}), it directly derives the uniform stability.

\section{4. Differentially Private Gradient Descent Ascent}

Among many optimization methods designed for the minimax problem, Gradient Descent Ascent (GDA) is one of the most widespread algorithm because of its simplicity, so we concentrate on GDA in this paper.

Let $\bold{w}_1,\bold{v}_1=\bold{0}$ be the intial model, $\eta_t$ be learning rate at iteration $t$ ($t=1,\cdots,T$), and ${\rm Proj}(\cdot)$ be projection to corresponding parameter spaces, then
\begin{equation*}
\begin{aligned}
\bold{w}_{t+1}&={\rm Proj}_\mathcal{W}\left(\bold{w}_t-\eta_t\nabla_\bold{w}L_S(\bold{w}_t,\bold{v}_t)\right), \\
\bold{v}_{t+1}&={\rm Proj}_\mathcal{V}\left(\bold{v}_t+\eta_t\nabla_\bold{v}L_S(\bold{w}_t,\bold{v}_t)\right).
\end{aligned}
\end{equation*}

To guarantee DP, we propose DP-GDA in Algorithm \ref{alg1}.
In Algorithm \ref{alg1}, the output is the average of iterates $\bar{\bold{w}}_T$ and $\bar{\bold{v}}_T$, rather than $\bold{w}_T$ and $\bold{v}_T$ themselves.
The reason is that the average operator simplifies the optimization error analysis \cite{nemirovski2009robust,lei2021stability}.
\begin{algorithm}[tb]
	\caption{Differentially Private Gradient Descent Ascent}
	\label{alg1}
	\begin{algorithmic}
		\State {\bfseries Input:} dataset $S$, privacy budgets $\epsilon,\delta$, learning rates $\eta_t$
		\State Initialize $\bold{w}_1,\bold{v}_1=\bold{0}$.
		\For{$i=1$ {\bfseries to} $T$}
		\State Sample $b_\bold{w},b_\bold{v}\sim\mathcal{N}\left(0,\sigma^2I_p\right)$.
		\State $\bold{w}_{t+1}={\rm Proj}_\mathcal{W}\left(\bold{w}_t-\eta_t\left(\nabla_\bold{w}L_S(\bold{w}_t,\bold{v}_t)+b_\bold{w}\right)\right)$ \\
		\State $\bold{v}_{t+1}={\rm Proj}_\mathcal{V}\left(\bold{v}_t+\eta_t\left(\nabla_\bold{v}L_S(\bold{w}_t,\bold{v}_t)+b_\bold{v}\right)\right)$
		\EndFor
		\State $\bar{\bold{w}}_T=\frac{1}{T}\sum_{t=1}^{T}\bold{w}_t, \quad \bar{\bold{v}}_T=\frac{1}{T}\sum_{t=1}^{T}\bold{v}_t.$
		\State Return $\mathcal{A}(S)=(\bar{\bold{w}}_T,\bar{\bold{v}}_T)$.
	\end{algorithmic}
\end{algorithm}

Before giving the privacy guarantees of Algorithm \ref{alg1}, we first recall the DP property in the single parameter setting.
\begin{lemma}\label{lem8}[\cite{wang2017differentially}]
	In single parameter DP gradient descent paradigm whose model updating process is $\bold{w}_{t+1}={\rm Proj}_\mathcal{W}\left(\bold{w}_t-\eta_t\left(\nabla_\bold{w}L_S(\bold{w}_t)+b_\bold{w}\right)\right)$\footnote{$L_S(\bold{w})$ here is the empirical risk in the single parameter setting.}, and the loss function is $G$-Lipschitz, for $\epsilon,\delta>0$, it is $(\epsilon,\delta)$-DP if the random noise is zero mean Gaussian noise, i.e., $b\sim\mathcal{N}(0,\sigma^2I_p)$, and for some constant $c$, $\sigma^2=c\frac{G^2T\log(1/\delta)}{n^2\epsilon^2}$.
\end{lemma}

By Lemma \ref{lem8}, we give the privacy guarantees of DP-GDA.
\begin{theorem}\label{the1}
	If $\ell(\cdot,\cdot;\cdot)$ satisfies Assumption \ref{a1}, then for $\epsilon,\delta>0$, DP-GDA is $(\epsilon,\delta)$-DP if
	\begin{equation*}
	\sigma=\mathcal{O}\left(\frac{G\sqrt{T\log(1/\delta)}}{n\epsilon}\right).
	\end{equation*}
\end{theorem}

Theorem \ref{the1} is directly derived from Lemma \ref{lem8}.
For differentially private gradient descent under single parameter condition, \cite{wang2017differentially} gives a tight noise bound (as shown in Lemma \ref{lem8}), via moments accountant theory \cite{abadi2016deep}.
In the minimax paradigm, it is also the gradient who may cause the privacy leakage when training, similar to the condition analyzed in \cite{wang2017differentially}.
And the privacy cost is independent of the minimization or maximization processes.
As a result, if the noise given in \cite{wang2017differentially} is injected to both $\bold{w}$ and $\bold{v}$, the claimed DP will be guaranteed.
This is also the reason that $\bold{w}$ and $\bold{v}$ share the same variance in Theorem \ref{the1}.
Besides, we average the parameters at the end of the algorithm, this will not effect the claimed DP because of the Post-Processing property of differential privacy \cite{dwork2014the}.
Considering that the proof process is almost the same (the only difference is to apply it to $\bold{v}$ once more) and we focus more on the generalization analysis in this paper, we directly use the result here in Algorithm \ref{alg1}.
For clarity, we give the proof in Appendix A.1.

\begin{remark}\label{rem2}
	In Algorithm \ref{alg1}, we apply gradient perturbation method to guarantee DP, rather than output or objective perturbation methods.
	One of the reasons is that gradient perturabtion naturally fits gradient-based algorithms (such as GDA), and as a result it can be used for a wide range of applications.
	Besides, adding random noise to the gradient allows the model to escape local minima \cite{raginsky2017nonconvex}.
	So we choose gradient perturbation method to guarantee DP in this paper.
\end{remark}

\section{5. Generalization Performance}

In this section, we analyze the stability of DP-GDA and then give corresponding generalization bounds.
To get the generalization bounds, we further assume that the loss function and the parameter space are bounded.
\begin{assumption}\label{a4}
	The loss function $\ell(\cdot,\cdot;\cdot)$ is assumed to be bounded, i.e., $0\leq\ell(\cdot,\cdot;\cdot)\leq M_\ell$.
\end{assumption}

\begin{assumption}\label{a5}
	Paramater spaces are assumed to be bounded: for all $\bold{w}$ and $\bold{v}$, $\|\bold{w}\|_2\leq M_\mathcal{W}$ and $\|\bold{v}\|_2\leq M_\mathcal{V}$.
\end{assumption}

\subsection{5.1. Stability Analysis}
Firstly, we analyze the argument stability (defined in Definition \ref{def2}) of our proposed DP-GDA.
\begin{theorem}\label{the2}
	If Assumptions \ref{a1} and \ref{a2} hold.
	Then with $\sigma$ given in Theorem \ref{the1} and $\eta_t=\frac{1}{\rho t}$, the output of DP-GDA (Algorithm 1) $\mathcal{A}(S)=(\bar{\bold{w}}_T,\bar{\bold{v}}_T)$ is $\gamma$-argument stable with probability at least $1-\zeta$ for $\zeta\in(\exp(-\frac{p}{8}),1)$, where
	\begin{equation*}
	\begin{aligned}
	\gamma&=\frac{4G}{n\rho}+\frac{2\sigma\sqrt{p}\log(eT)}{T}p_\zeta+4\sqrt{\log(eT)}\times \\
	&\quad\sqrt{\frac{G^2}{\rho^2T}+\frac{\sigma^2p}{\rho^2T}p_\zeta^2+\frac{2G\sigma\sqrt{p}}{\rho^2T}p_\zeta+\frac{\left(g_\bold{w}+g_\bold{v}\right)\sigma\sqrt{p}}{\rho\log(eT)}p_\zeta},
	\end{aligned}
	\end{equation*}
\end{theorem}
where $p_\zeta=1+\big(\frac{8\log(2T/\zeta)}{p}\big)^{1/4}$, $g_\bold{w}=\left\Vert\bar{\bold{w}}^*-\bar{\bold{w}}_{T}\right\Vert_2$, and $g_\bold{v}=\left\Vert\bar{\bold{v}}^*-\bar{\bold{v}}_{T}\right\Vert_2$ for $\bar{\bold{w}}^*=\arg\min_{\bold{w}\in\mathcal{W}}L_S(\bold{w},\bar{\bold{v}}_T)$, and $\bar{\bold{v}}^*=\arg\max_{\bold{v}\in\mathcal{V}}L_S(\bar{\bold{w}}_T,\bold{v})$.

The proof can be found in Appendix A.2.
We first compare it with non-DP minimax problem.
The key difference is that in DP-GDA, random noise is an essential part, which brings challenges to the theoretical analysis if we want an acceptable stability bound.
In the non-DP setting, if we set $T$ larger, the stability bound will be better in general, however, in DP-GDA, this is not true because the $T$-time injected random noise also affects the stability.
Then, we compare it with traditional single parameter problem under DP condition.
When it comes to the DP setting, there exist terms $\|b\|_2$ and $\|b\|_2^2$ (in the minimax problem, they are divided into $b_\bold{w}$ and $b_\bold{v}$).
Among them, $\|b\|_2^2$ is an \textit{acceptable} term because the variance of $b$ is of the order $\mathcal{O}(1/n)$ w.r.t $n$, so the key challenge is to bound $\|b\|_2$, especially for high probability bounds (in expectation bounds, $\mathbb{E}[\|b\|_2]=0$).
In traditional DP settings, factor $\|b\|_2$ can be eliminated by choosing particular $\eta_t$ (for example, in \cite{wang2017differentially}, term $\|b\|_2$ vanishes by setting $\eta_t=1/L$, where $L$ is the smoothness parameter).
However, in the minimax paradigm, this cannot go through because parameter $\bold{w}$ differs with the changing of $\bold{v}$ (and vice versa), which brings huge troubles to the theoretical analysis.
To solve the problems mentioned above, we introduce terms $g_\bold{w},g_\bold{v}$ to bound $\|b_\bold{w}\|_2$ and $\|b_\bold{v}\|_2$, respectively.
Meanwhile, if we bound $g_\bold{w},g_\bold{v}$ rudely by $M_\mathcal{W}$ and $M_\mathcal{V}$, the result will be worse, detailed discussions are shown in Remark \ref{rem4}.

\begin{remark}\label{rem3}
	In Theorem \ref{the2}, there exist terms $g_\bold{w}$ and $g_\bold{v}$, we discuss them here.
	With $\rho$-strongly convexity, we have $\rho g_\bold{w}^2\leq2(L_S\left(\bar{\bold{w}}_{t},\bar{\bold{v}}_T\right)-L_S\left(\bar{\bold{w}}^*,\bar{\bold{v}}_T\right)) $.
	By extending the classical analysis of the gradient descent model (see e.g. \cite{shalev2014understanding}), if Assumptions \ref{a1}, \ref{a2} and \ref{a5} hold, then with $\eta_t=\frac{1}{\rho t}$, we have
	\begin{equation*}
	\begin{aligned}
	&L_S\left(\bar{\bold{w}}_{t},\bar{\bold{v}}_T\right)-L_S\left(\bar{\bold{w}}^*,\bar{\bold{v}}_T\right) \\
	&\leq\log(eT)\left(\frac{G^2}{2\rho T}+\left(\frac{G}{T}+\frac{M_\mathcal{W}}{\log(eT)}\right)\|b_\bold{w}\|_2+\frac{\|b_\bold{w}\|_2^2}{2T}\right).
	\end{aligned}
	\end{equation*}
	So with probability at least $1-\zeta$ for $\zeta\in(\exp(-\frac{p}{8}),1)$,
	\begin{equation*}
	\begin{aligned}
	g_\bold{w}&\leq\sqrt{\log(eT)}\times \\
	&\quad\sqrt{\frac{G^2}{\rho^2T}+\frac{2}{\rho}\left(\frac{G\sigma\sqrt{p}}{T}+\frac{M_\mathcal{W}\sigma\sqrt{p}}{\log(eT)}\right)p_\zeta'+\frac{\sigma^2p}{\rho T}p_\zeta'^2},
	\end{aligned}
	\end{equation*}
	where $p_\zeta'=1+\big(\frac{8\log(T/\zeta)}{p}\big)^{1/4}$.
	
	Similarly, under Assumptions \ref{a1}, \ref{a2} and \ref{a5}, $g_\bold{v}$ shares the same property with high probability.
	Taking $\sigma$ given in Theorem \ref{the1} and omitting $\log(\cdot)$ terms, we have
	\begin{equation*}
	g_\bold{w},g_\bold{v}=\mathcal{O}\left(\sqrt{\frac{1}{T}+\frac{\sqrt{p}}{n\sqrt{T}\epsilon}+\frac{\sqrt{pT}}{n\epsilon}+\frac{\sqrt{p}}{n^2\epsilon^2}}\right).
	\end{equation*}
	
	If taking $T=\mathcal{O}(n^{2/3})$, then with high probability, we have $g_\bold{w},g_\bold{v}=\mathcal{O}\big(\frac{p^{1/4}}{n^{1/3}\epsilon^{1/2}}\big)$.
\end{remark}

\begin{remark}\label{rem4}
	Taking $\sigma$ in Theorem \ref{the2} and $T=\mathcal{O}(n^{2/3})$, then with high probability, we have $\gamma=\mathcal{O}\big(\frac{p^{1/4}}{n^{1/2}\epsilon^{1/2}}\big)$.
	
	Then we can answer the question left above, if we do not bound term $\|b_\bold{w}\|_2$ and $\|b_\bold{v}\|_2$ with $g_\bold{w},g_\bold{v}$, but `rudely' using $M_\mathcal{W}$ and $M_\mathcal{V}$, then the stability bound comes to
	\begin{equation*}
	\begin{aligned}
	\gamma&=\frac{4G}{n\rho}+\frac{2\sigma\sqrt{p}\log(eT)}{T}p_\zeta+4\sqrt{\log(eT)}\times \\
	&\sqrt{\frac{G^2}{\rho^2T}+\frac{\sigma^2p}{\rho^2T}p_\zeta^2+\frac{2G\sigma\sqrt{p}}{\rho^2T}p_\zeta+\frac{\left(M_\mathcal{W}+M_\mathcal{V}\right)\sigma\sqrt{p}}{\rho\log(eT)}p_\zeta}.
	\end{aligned}
	\end{equation*}
	
	Under this condition, if we take $T=\mathcal{O}(n^{2/3})$ and Assumption \ref{a5} holds, the stability bound comes to $\gamma=\mathcal{O}\big(\frac{p^{1/4}}{n^{1/3}\epsilon^{1/2}}\big)$.
	Thus, by introducing $g_\bold{w}$ and $g_\bold{v}$, we improve the stability bound from $\mathcal{O}\big(\frac{1}{n^{1/3}}\big)$ to $\mathcal{O}\big(\frac{1}{n^{1/2}}\big)$.
\end{remark}

\subsection{5.2. Utility Bounds via Stability}
In this section, we connect different generalization measures listed in Definition \ref{def1} with the argument stability and give corresponding risk bounds.

\begin{theorem}\label{the3} With argument stability parameter $\gamma$,
	
	(a) If Assumptions \ref{a2} and \ref{a4} hold, then for all $\iota,\zeta>0$, with probability at least $1-\zeta$, the plain generalization error satisfies
	\begin{equation*}
	\begin{aligned}
	&L(\mathcal{A}_\bold{w}(S),\mathcal{A}_\bold{v}(S))-\frac{1}{1-\iota}L_S(\mathcal{A}_\bold{w}(S),\mathcal{A}_\bold{v}(S)) \\
	&\leq\sqrt{\frac{\left(G^2\gamma^2+64G^2n\gamma^2\log\left(3/\zeta\right)\right)}{2\left(1-\iota\right)^2n}\log\left(\frac{3}{\zeta}\right)} \\
	&\quad+\frac{50\sqrt{2}eG\gamma\log(n)}{1-\iota}\log\left(\frac{3e}{\zeta}\right) \\
	&\quad+\frac{\left(12+2\iota\right)M_\ell}{3\iota\left(1-\iota\right)n}\log\left(\frac{3}{\zeta}\right).
	\end{aligned}
	\end{equation*}
	
	(b) If Assumptions \ref{a1}, \ref{a2}, \ref{a3}, and \ref{a4} hold, then for all $\iota,\zeta>0$, with probability at least $1-\zeta$, the primal generalizaiton error staisfies
	\begin{equation*}
	\begin{aligned}
	&R\left(\mathcal{A}_\bold{w}(S)\right)-\frac{1}{1-\iota}R_S\left(\mathcal{A}_\bold{w}(S)\right) \\
	&\leq\sqrt{\frac{\left(1+L/\rho\right)^2G^2\gamma^2\left(1+64n\log\left(3/\zeta\right)\right)}{2(1-\iota)^2n}\log\left(\frac{3}{\zeta}\right)} \\
	&\quad+\frac{50\sqrt{2}\left(1+L/\rho\right)G\gamma\log(n)}{1-\iota}\log\left(\frac{3e}{\zeta}\right) \\
	&\quad+\frac{(12+2\iota)M_\ell}{3\iota(1-\iota)n}\log\left(\frac{3}{\zeta}\right).
	\end{aligned}
	\end{equation*}
	
	(c) If Assumptions \ref{a1}, \ref{a2}, \ref{a3}, and \ref{a4} hold, then for all $\iota,\zeta>0$, with probability at least $1-\zeta$, the primal excess population risk satisfies
	\begin{equation*}
	\begin{aligned}
	&R\left(\mathcal{A}_\bold{w}(S)\right)-\frac{1+\iota}{1-\iota}\inf_{\bold{w}\in\mathcal{W}}R\left(\bold{w}\right) \\
	&\leq\sqrt{\frac{\left(1+L/\rho\right)^2G^2\gamma^2\left(1+64n\log\left(6/\zeta\right)\right)}{2(1-\iota)^2n}\log\left(\frac{6}{\zeta}\right)} \\
	&\quad+\sqrt{\frac{\left(G^2\gamma^2+64G^2n\gamma^2\log\left(6/\zeta\right)\right)}{2(1-\iota)^2n}\log\left(\frac{6}{\zeta}\right)} \\
	&\quad+\frac{50\sqrt{2}\left(1+e+L/\rho\right)G\gamma\log(n)}{1-\iota}\log\left(\frac{6e}{\zeta}\right) \\
	&\quad+\frac{(24+4\iota)M_\ell}{3\iota(1-\iota)n}\log\left(\frac{6}{\zeta}\right) \\
	&\quad+\frac{1}{1-\iota}\triangle_S^s\left(\mathcal{A}_\bold{w}(S),\mathcal{A}_\bold{v}(S)\right).
	\end{aligned}
	\end{equation*}
	
	(d) If Assumptions \ref{a1}, \ref{a2}, \ref{a3}, and \ref{a4} hold, then for all $\iota,\zeta>0$, with probability at least $1-\zeta$, the strong primal dual population risk satisfies\footnote{For simplicity, we couple $\triangle_S^s\left(\mathcal{A}_\bold{w}(S),\mathcal{A}_\bold{v}(S)\right)$ and its expectation together here, since there is a `global' upper bound for strong PD emprical risk. More details are shown in Appendix A.3.}
	\begin{equation*}
	\begin{aligned}
	&\triangle^s\left(\mathcal{A}_\bold{w}(S),\mathcal{A}_\bold{v}(S)\right) \\
	&\leq\frac{100\sqrt{2}e(1+\iota)(1+L/\rho)G\gamma\log(n)}{1-\iota}\log\left(\frac{e}{\zeta}\right) \\
	&\quad+\frac{144e(1+\iota)G^2}{\rho\iota(1-\iota)n}\log\left(\frac{e}{\zeta}\right)+\frac{8e(1+\iota)M_\ell}{n(1-\iota)}\log\left(\frac{e}{\zeta}\right) \\
	&\quad+\left(\frac{e\iota}{1-\iota}\log\left(\frac{e}{\zeta}\right)+1\right)\triangle_S^s\left(\mathcal{A}_\bold{w}(S),\mathcal{A}_\bold{v}(S)\right).
	\end{aligned}
	\end{equation*}
\end{theorem}

\begin{remark}
	For part (a), the plain generalization error is essentially analyzed via uniform stability, defined as the upper bound of $\ell\left(\mathcal{A}_\bold{w}(S),\mathcal{A}_\bold{v}(S);z\right)-\ell\left(\mathcal{A}_\bold{w}(S'),\mathcal{A}_\bold{v}(S');z\right)$.
	If the loss function $\ell(\cdot,\cdot;\cdot)$ is $G$-Lipschitz (Assumption \ref{a2}), $\gamma$-argument stability derives $G\gamma$-uniform stability (details can be found in (\ref{eq32}) in Appendix A.3).
	Then the proof of part (a) is completed by the uniform stability of model $(\mathcal{A}_\bold{w}(S),\mathcal{A}_\bold{v}(S))$.
	For parts (b), (c), and (d), argument stability is applied to get the claimed results.
	The proof can be found in Appendix A.3.
\end{remark}

We give the bound of $\triangle_S^s\left(\mathcal{A}_\bold{w}(S),\mathcal{A}_\bold{v}(S)\right)$ in the following and then discuss the results given in Theorem \ref{the3}.

\begin{lemma}\label{lem9}
	If Assumptions \ref{a1} and \ref{a2} hold.
	Taking $\sigma$ given in Theorem \ref{the1}, and $\eta_t=\frac{1}{\rho t}$, then with probability at least $1-\zeta$ for $\zeta\in(\exp(-\frac{p}{8}),1)$, the strong primal dual empirical risk of the output of DP-GDA: $\mathcal{A}(S)=(\bar{\bold{w}}_T,\bar{\bold{v}}_T)$ satisfies
	\begin{equation*}
	\begin{aligned}
	&\triangle_S^s\left(\bar{\bold{w}}_T,\bar{\bold{v}}_T\right) \\
	&\leq \frac{G^2\log(eT)}{\rho T}+\frac{cG\left(g_\bold{w}+g_\bold{v}\right)\sqrt{Tp\log(1/\delta)}}{n\epsilon}p_\zeta \\
	&\quad+cG^2\log(eT)\left(\frac{p\log(1/\delta)}{\rho n^2\epsilon^2}p_\zeta^2+\frac{2\sqrt{p\log(1/\delta)}}{\rho\sqrt{T}n\epsilon}p_\zeta\right),
	\end{aligned}
	\end{equation*}
	for some constant $c$, where $p_\zeta,g_\bold{w},g_\bold{v}$ are defined as in Theorem \ref{the2}.
\end{lemma}
The proof of Lemma \ref{lem9} is a vital part of Theorem \ref{the2}.
For clarity, we give a brief sketch in Appendix A.4.

Given the strong PD empirical risk $\triangle_S^s\left(\bar{\bold{w}}_T,\bar{\bold{v}}_T\right)$ and the argument stability parameter $\gamma$ of the output of DP-GDA: $(\bar{\bold{w}}_T,\bar{\bold{v}}_T)$, we now give some discussions of Theorem \ref{the3} in Remark \ref{rem5}.
\begin{remark}\label{rem5}
	Part (a) connects the argument stability with the plain generalization error, for the output of Algorithm \ref{alg1}, if taking $T=\mathcal{O}(n^{2/3})$, we have $\gamma=\mathcal{O}\left(\frac{p^{1/4}}{n^{1/2}\epsilon^{1/2}}\right)$
	as discussed in Remark \ref{rem4}.
	Then plugging this result back into the plain generalization error, with high probability, we can bound $L(\bar{\bold{w}}_T,\bar{\bold{v}}_T)$ by
	\begin{equation*}
	\frac{1}{1-\iota}\mathcal{O}\left(L_S(\bar{\bold{w}}_T,\bar{\bold{v}}_T)+\frac{p^{\frac{1}{4}}\log(n)}{n^{\frac{1}{2}}\epsilon^{\frac{1}{2}}}\right).
	\end{equation*} 
	For a well trained model (with $n$ large enough), the empirical risk $L_S(\bar{\bold{w}}_T,\bar{\bold{v}}_T)$ can be relatively small \cite{lever2013tighter,yang2019fast,cortes2021relative}.
	As a result, if $L_S(\bar{\bold{w}}_T,\bar{\bold{v}}_T)=\mathcal{O}\left(1/n^{1/2}\right)$, the plain generalization error along with the plain population risk can be bounded by $\mathcal{O}\left(\frac{p^{1/4}}{n^{1/2}\epsilon^{1/2}}\right)$ if $\log(\cdot)$ terms are omitted.
	
	Part (b) connects the argument stability with the primal generalization error.
	Like discussed above, taking $\gamma$ when $T=\mathcal{O}(1/n^{2/3})$, $R(\bar{\bold{w}}_T)$ can be bounded by
	\begin{equation*}
	\frac{1}{1-\iota}\mathcal{O}\left(R_S(\bar{\bold{w}}_T)+\frac{p^{\frac{1}{4}}\log(n)}{n^{\frac{1}{2}}\epsilon^{\frac{1}{2}}}\right)
	\end{equation*}
	with high probabilty.
	Thus, for well trained model, the primal generalization error along with the primal population risk can be bounded by $\mathcal{O}\left(\frac{p^{1/4}}{n^{1/2}\epsilon^{1/2}}\right)$.
	
	Part (c) connects the argument stability with the primal excess population risk.
	Here, we still taking $T=\mathcal{O}(1/n^{2/3})$.
	Under this circumstance, the strong PD empirical risk of $(\bar{\bold{w}}_T,\bar{\bold{v}}_T)$ can be bounded by
	\begin{equation}\label{eq36}
	\mathcal{O}\left(\frac{p^{\frac{3}{4}}}{n\epsilon^{\frac{3}{2}}}+\frac{1}{n^{\frac{2}{3}}}\right).
	\end{equation}
	As a result, the primal excess population risk can be bounded by
	\begin{equation*}
	\frac{1}{1-\iota}\mathcal{O}\left(\frac{p^{\frac{1}{4}}}{n^{\frac{1}{2}}\epsilon^{\frac{1}{2}}}+\frac{p^{\frac{3}{4}}}{n\epsilon^{\frac{3}{2}}}+\frac{1}{n^{\frac{2}{3}}}+\inf_{\bold{w}\in\mathcal{W}}R(w)\right).
	\end{equation*}
	If $\inf_{\bold{w}\in\mathcal{W}}R(w)$ is smaller than the maximum value of the other three terms, the primal population risk can be bounded by
	\begin{equation*}
	\mathcal{O}\left(\max\left\{\frac{p^{\frac{1}{4}}}{n^{\frac{1}{2}}\epsilon^{\frac{1}{2}}},\frac{p^{\frac{3}{4}}}{n\epsilon^{\frac{3}{2}}},\frac{1}{n^{\frac{2}{3}}}\right\}\right).
	\end{equation*}
	
	Part (d) connects the argument stability with the strong PD population risk.
	Like discussed above, if we take $T=\mathcal{O}(1/n^{2/3})$, the strong PD empirical risk of $(\bar{\bold{w}}_T,\bar{\bold{v}}_T)$ shares the same property as in (\ref{eq36}), and the strong PD population risk can be bounded by
	\begin{equation*}
	\mathcal{O}\left(\max\left\{\frac{p^{\frac{1}{4}}}{n^{\frac{1}{2}}\epsilon^{\frac{1}{2}}},\frac{p^{\frac{3}{4}}}{n\epsilon^{\frac{3}{2}}},\frac{1}{n^{\frac{2}{3}}}\right\}\right).
	\end{equation*}
\end{remark}

\begin{remark}\label{rem7}
	Here, we discuss coefficient $\iota$ in Theorem \ref{the3}.
	We first explain why $\iota$ exists.
	The existence of $\iota$ is beacuse when decomposing term $\sqrt{L(\bold{w},\bold{v})/n}$, inequality $\sqrt{ab}\leq\iota a+b/\iota$ was applied.
	Then we have\footnote{We omit other terms here (such as $\log(1/\zeta)$ and $M_\ell$), details can be found in the Appendix.} $\sqrt{L(\bold{w},\bold{v})/n}\leq\iota L(\bold{w},\bold{v})+1/(\iota n)$.
	In this way, we decouple $n$ from the square root, in order to get sharper bounds when connecting stability with the generalization error (the improvement is from $\mathcal{O}(\sqrt{1/n})$ to $\mathcal{O}(1/n)$), especially when the model is well behaved, $L(\bold{w},\bold{v})$ is a small value.
	Besides, when it comes to analyzing sharp generalization error, $\iota$ commonly appears.
	For clarity, in the following, we represent the population risk and the empirical risk by $Pf$ and $P_nf$, respectively, like in \cite{bartlett2005local}.
	\cite{bartlett2005local} proposes the generalization error like $Pf\leq\mathcal{O}\big(\iota(\frac{1}{\iota-1}P_nf+r^*+\frac{\log(1/\zeta)}{n})\big)$; \cite{klochkov2021stability} proposes the generalization error like $Pf\leq\mathcal{O}\big((1+\iota)P_nf+\frac{\log(1/\zeta)}{\iota n}\big)$.
	These results are similar to ours by rearrangement, differences are caused by choosing different $\iota$.
	Results in the similar form include \cite{catoni2007pac}, \cite{lever2013tighter}, \cite{yang2019fast}, to mention but a few.
	When the machine learning model is well trained and the empirical risk is small, generalization error of this form is better \cite{bartlett2005local,lever2013tighter,yang2019fast,klochkov2021stability,cortes2021relative}.
\end{remark}

According to the property of the weak PD population risk, we can directly get the following corollary via Theorem \ref{the3}.

\begin{table*}[t]
	\caption{Utility bounds of our proposed DP-GDA method.}
	\label{tab1}
	\vskip 0.15in
	\begin{center}
		\begin{small}
			\begin{sc}
				\begin{tabular}{ccccc}
					\toprule
					Generalization Measures & $\rho$ & $G$ & $L$ & Utility Bounds \\
					\midrule
					Weak PD Population Risk / Generalization Error & $\surd$ & $\surd$ & $\surd$ & $\mathcal{O}\left(\max\left\{\frac{p^{1/4}}{n^{1/2}\epsilon^{1/2}},\frac{p^{3/4}}{n\epsilon^{3/2}},\frac{1}{n^{2/3}}\right\}\right)$ \\
					Strong PD Population Risk / Generalization Error & $\surd$ & $\surd$ & $\surd$ & $\mathcal{O}\left(\max\left\{\frac{p^{1/4}}{n^{1/2}\epsilon^{1/2}},\frac{p^{3/4}}{n\epsilon^{3/2}},\frac{1}{n^{2/3}}\right\}\right)$ \\
					Primal Excess Population Risk & $\surd$ & $\surd$ & $\surd$ & $\mathcal{O}\left(\max\left\{\frac{p^{1/4}}{n^{1/2}\epsilon^{1/2}},\frac{p^{3/4}}{n\epsilon^{3/2}},\frac{1}{n^{2/3}}\right\}\right)$ \\
					Primal Population Risk / Generalization Error & $\surd$ & $\surd$ & $\surd$ & $\mathcal{O}\left(\frac{p^{1/4}}{n^{1/2}\epsilon^{1/2}}\right)$ \\
					Plain Population Risk / Generalization Error & $\surd$ & $\surd$ & $\times$ & $\mathcal{O}\left(\frac{p^{1/4}}{n^{1/2}\epsilon^{1/2}}\right)$ \\
					\bottomrule
				\end{tabular}
			\end{sc}
		\end{small}
	\end{center}
	\vskip -0.1in
\end{table*}

\begin{corollary}\label{cor1}
	With argument stability parameter $\gamma$,
	
	(a) Under the condition given in Theorem \ref{the3} part (d), the strong PD generalization error satisfies
	\begin{equation*}
	\begin{aligned}
	&\triangle^s\left(\mathcal{A}_\bold{w}(S),\mathcal{A}_\bold{v}(S)\right)-\triangle_S^s\left(\mathcal{A}_\bold{w}(S),\mathcal{A}_\bold{v}(S)\right) \\
	&\leq\frac{100\sqrt{2}e(1+\iota)(1+L/\rho)G\gamma\log(n)}{1-\iota}\log\left(\frac{e}{\zeta}\right) \\
	&\quad+\frac{144e(1+\iota)G^2}{\rho\iota(1-\iota)n}\log\left(\frac{e}{\zeta}\right)+\frac{8e(1+\iota)M_\ell}{n(1-\iota)}\log\left(\frac{e}{\zeta}\right) \\
	&\quad+\frac{e\iota}{1-\iota}\log\left(\frac{e}{\zeta}\right)\mathbb{E}_S\left[\triangle_S^s\left(\mathcal{A}_\bold{w}(S),\mathcal{A}_\bold{v}(S)\right)\right].
	\end{aligned}
	\end{equation*}
	
	(b) Under the condition given in Theorem \ref{the3} part (d), the weak primal dual population risk satisfies
	\begin{equation*}
	\begin{aligned}
	&\triangle^w\left(\mathcal{A}_\bold{w}(S),\mathcal{A}_\bold{v}(S)\right) \\
	&\leq\frac{100\sqrt{2}e(1+\iota)(1+L/\rho)G\gamma\log(n)}{1-\iota}\log\left(\frac{e}{\zeta}\right) \\
	&\quad+\frac{144e(1+\iota)G^2}{\rho\iota(1-\iota)n}\log\left(\frac{e}{\zeta}\right)+\frac{8e(1+\iota)M_\ell}{n(1-\iota)}\log\left(\frac{e}{\zeta}\right) \\
	&\quad+\left(\frac{e\iota}{1-\iota}\log\left(\frac{e}{\zeta}\right)+1\right)\triangle_S^s\left(\mathcal{A}_\bold{w}(S),\mathcal{A}_\bold{v}(S)\right).
	\end{aligned}
	\end{equation*}
	
	(c) Under the condition given in Theorem \ref{the3} part (d), the weak PD generalization error satisfies
	\begin{equation*}
	\begin{aligned}
	&\triangle^w\left(\mathcal{A}_\bold{w}(S),\mathcal{A}_\bold{v}(S)\right)-\triangle_S^w\left(\mathcal{A}_\bold{w}(S),\mathcal{A}_\bold{v}(S)\right) \\
	&\leq\frac{100\sqrt{2}e(1+\iota)(1+L/\rho)G\gamma\log(n)}{1-\iota}\log\left(\frac{e}{\zeta}\right) \\
	&\quad+\frac{144e(1+\iota)G^2}{\rho\iota(1-\iota)n}\log\left(\frac{e}{\zeta}\right)+\frac{8e(1+\iota)M_\ell}{n(1-\iota)}\log\left(\frac{e}{\zeta}\right) \\
	&\quad+\left(\frac{e\iota}{1-\iota}\log\left(\frac{e}{\zeta}\right)+2\right)\triangle_S^s\left(\mathcal{A}_\bold{w}(S),\mathcal{A}_\bold{v}(S)\right).
	\end{aligned}
	\end{equation*}
\end{corollary}

Corollary \ref{cor1} can be easily extended from Theorem \ref{the3} part (d).
For clarity, we give proof sketches in Appendix A.5.
Meanwhile, the strong and weak PD generalization errors and the weak PD population risk share similar properties with the strong PD population risk given in Thorem \ref{the3} part (d), as discussed in Remark \ref{rem5}.

\begin{remark}\label{rem6}
	By the discussions given in Remark \ref{rem5}, we find that the bottleneck of the generalization bounds is the injected random noise, rather than the analysis apporach itself.
	The reason is that we apply novel decomposition methods (motivated by \cite{klochkov2021stability}) and overcome the $\mathcal{O}(1/\sqrt{n})$ terms when connecting the stability with the generalization error.
	However, when it comes to DP paradigm, when analyzing $\gamma$, there exists term\footnote{We pay our attentions to the balance between $T$ and $n$ here and omit parameters $p,\epsilon$.}
	\begin{equation*}
	\mathcal{O}\left(\sqrt{\frac{\sqrt{T}}{n}\sqrt{\frac{1}{T}}}\right),
	\end{equation*}
	in which term $\sqrt{T}/n$ is derived from the standard deviation $\sigma$ of the random noise and term $\sqrt{1/T}$ is derived from the optimization error.
	Together with the discussions given before, in the minimax problem, it is hard to eliminate term $\|b\|_2$ because $\bold{w}$ differs when $\bold{v}$ changes.
	So the best result is $\mathcal{O}(1/\sqrt{n})$, under the setting of gradient perturbation concentrated by this paper, and this may give inspirations to other researchers on how to get better generalization performance in the DP-minimax condition.
\end{remark}

\section{6. Comparisons with Related Work}

In this section, we compare our given bounds with previous related work.
The generalization bounds under corresponding assumptions are listed in detail in Table \ref{tab1}, in which $\rho,G,L$ represents $\rho$-SC-SC, $G$-Lipschitz, and $L$-smooth, respectively.
Considering there is no existed work analyzing the theoretical bounds for the general DP-minimax problem, so we only give our results in Table \ref{tab1}.

\subsection{6.1. Non-DP General Minimax Setting}
For the general minimax problem, DP has not been applied to the best of our knowledge so we compare our results with previous non-DP results \cite{farnia2021train,lei2021stability,zhang2021generalization}.
Among them, \cite{zhang2021generalization} focuses on ESP and \cite{lei2021stability} focuses on SGDA and AGDA, so the only existed result for GDA is given in \cite{farnia2021train}, it is an $\mathcal{O}\big(1/n\big)$ primal generalization error for GDA in expectation.

\subsection{6.2. DP-Minimax Settings}
For DP-minimax settings, to the best of our knowledge, the theoretical results are only explicitly given for DP-AUC maximization:
\cite{wang2021differentially} gives the expectation excess population risk bounds for DP-AUC maximization problems under output perturbation and objective perturbation methods.
The results are of the orders $\mathcal{O}\big(\max\big\{p^{1/3}/(n\epsilon)^{2/3},1/\sqrt{n}\big\}\big)$ and $\mathcal{O}\big(\max\big\{\sqrt{p}/(\sqrt{n}\epsilon^2),1/n^{1/3}\big\}\big)$ for output and objective perturbations, respectively.
Considering that it is hard to compare high probability bounds with expectation bounds, and the perturbation methods are different, we only list them here for comparisons.
\cite{huai2020pairwise,yang2021stability} analyzes the high probability excess population risk for DP-AUC maximization problems under the pairswise learning paradigm, whose bounds are of the order $\mathcal{O}\big(\sqrt{p}/\sqrt{n}\epsilon\big)$.
In this paper, corresponding generalization measure is the plain population risk, so our result is better than previous ones, by an order of $\mathcal{O}\big(p^{1/4}/\epsilon^{1/2}\big)$, considering $\epsilon$ is always set smaller than 1 for meaningful DP.

\section{7. Conclusion}

In this paper, we focus on the differential privacy of general minimax paradigm and propose DP-GDA.
Except for the privacy guarantees, we analyze the stability of DP-GDA and connect it with the generalization performance.
Furthermore, we give corresponding bounds under different generalization measures and compare them with previous works on differetially private particular minimax models, theoretical results show that our generalization bounds are better.
We believe that our analysis and discussions will give inspirations to DP-minimax generalization performance researches.
In future work, we will attempt to relax the assumptions needed in this paper, and overcome challenges brought by the injected random noise, in order to give sharper generalization bounds for general DP-minimax paradigm.

\bibliographystyle{aaai}
\bibliography{icml2022}

\onecolumn
\begin{appendix}

\section{A. Details of proofs}

Recall that in DP-GDA, the training process is
\begin{equation*}
\begin{aligned}
\bold{w}_{t+1}&=\bold{w}_t-\eta_t\left(\nabla_\bold{w}L_S(\bold{w}_t,\bold{v}_t)+b_\bold{w}\right), \\
\bold{v}_{t+1}&=\bold{v}_t+\eta_t\left(\nabla_\bold{v}L_S(\bold{w}_t,\bold{v}_t)+b_\bold{v}\right),
\end{aligned}
\end{equation*}
where $b_\bold{w},b_\bold{v}\sim\mathcal{N}\left(0,\sigma^2I_p\right)$ and $\sigma=c\frac{G\sqrt{T\log(1/\delta)}}{n\epsilon}$.

And the output is the average of iterates, defined as $\bar{\bold{w}}_T=\frac{1}{T}\sum_{t=1}^{T}\bold{w}_t, \quad \bar{\bold{v}}_T=\frac{1}{T}\sum_{t=1}^{T}\bold{v}_t$.

\subsection{A.1. Proof of Theorem \ref{the1}}

Before detailed proof, we first revisit the moments accountant method \cite{abadi2016deep}.

\begin{definition}\label{def5}[Privacy loss \cite{abadi2016deep}]
	For adjacent datasets $D,D'$, mechansim $\mathcal{M}$ and an output $o\in\mathbb{R}$, the privacy loss at $o$ is defined as:
	\begin{equation*}
	c(o;\mathcal{M},D,D')=\log\left(\frac{\mathbb{P}\left[\mathcal{M}(D)=o\right]}{\mathbb{P}\left[\mathcal{M}(D')=o\right]}\right).
	\end{equation*}	
\end{definition}

\begin{definition}\label{def6}[Moment \cite{abadi2016deep}]
	For given mechanism $\mathcal{M}$ and the privacy loss at output $o$, the $\lambda^{th}$ moment is defined as:
	\begin{equation*}
	\alpha_\mathcal{M}(\lambda;D,D')=\log\left(\mathbb{E}_{o\sim\mathcal{M}(D)}\left[\exp\left(\lambda c(o;\mathcal{M},D,D')\right)\right]\right),
	\end{equation*}
	whose upper bound is defined as:
	\begin{equation*}
	\alpha_\mathcal{M}(\lambda)=\max_{D,D'}\alpha_\mathcal{M}(\lambda;D,D').
	\end{equation*}
\end{definition}

\begin{lemma}\label{lem10}[Composability \cite{abadi2016deep}]
	Let $\alpha_\mathcal{M}(\lambda)$ defined as above and suppose $\mathcal{M}$ consists of several mechanisms $\mathcal{M}_1,\cdots,\mathcal{M}_k$ and $\mathcal{M}_i$ relies on $\mathcal{M}_1,\cdots,\mathcal{M}_{i-1}$. Then for any $\lambda$:
	\begin{equation*}
	\alpha_\mathcal{M}(\lambda)\leq\sum_{i=1}^{k}\alpha_{\mathcal{M}_i}(\lambda).
	\end{equation*}
\end{lemma}

\begin{lemma}\label{lem11}[Tail Bound \cite{abadi2016deep}]
	Let $\alpha_\mathcal{M}(\lambda)$ defined as above, for any $\epsilon>0$, $\mathcal{M}$ is ($\epsilon,\delta$)-differential privacy if
	\begin{equation*}
	\delta=\min_\lambda\exp\left(\alpha_\mathcal{M}(\lambda)-\lambda\epsilon\right).
	\end{equation*}
\end{lemma}

\begin{definition}\label{def7}[R{\'e}nyi Divergence \cite{bun2016concentrated}]
	Let $P$ and $Q$ be probability distributions, for $\kappa\in(1,\infty)$, the R{\'e}nyi Divergence of order $\kappa$ between $P$ and $Q$ is defined as:
	\begin{equation*}
	D_\kappa(P||Q)=\frac{1}{\kappa-1}\log\left(\mathbb{E}_{x\sim P}\left[\left(\frac{P(x)}{Q(x)}\right)^{\kappa-1}\right]\right).
	\end{equation*}
\end{definition}

\begin{lemma}\label{lem12}
	Let $\mu,\nu\in\mathbb{R}^p$, $\sigma\in\mathbb{R}$, and $\kappa\in(1,\infty)$, then for Gaussian distribution $\mathcal{N}(\cdot,\cdot)$, we have
	\begin{equation*}
	D_\kappa\left(\mathcal{N}\left(\mu,\sigma^2I_p\right)||\mathcal{N}\left(\nu,\sigma^2I_p\right)\right)=\frac{\kappa\|\mu-\nu\|_2^2}{2\sigma^2}.
	\end{equation*}
\end{lemma}

Then, we give the detailed proof.

\begin{theorem}
	If $\ell(\cdot,\cdot;\cdot)$ satisfies Assumption \ref{a1}, then for $\epsilon,\delta>0$, DP-GDA is $(\epsilon,\delta)$-DP if
	\begin{equation*}
	\sigma=\mathcal{O}\left(\frac{G\sqrt{T\log(1/\delta)}}{n\epsilon}\right).
	\end{equation*}
\end{theorem}

\begin{proof}
	We first analyze parameter $\bold{w}$.
	
	When updating $\bold{w}$, at iteration $t$, the randomized mechanism $\mathcal{M}_t$ which may disclose privacy is
	\begin{equation*}
	\mathcal{M}_t=\nabla_\bold{w}L_S(\bold{w}_t,\bold{v}_t)+b_\bold{w}=\frac{1}{n}\sum_{i=1}^n\nabla_\bold{w}\ell(\bold{w}_t,\bold{v}_t;z_i)+b_\bold{w}.
	\end{equation*}
	
	Denote prbability distribution of $M_t$ over adjacent datasets $D,D'$ as $P$ and $Q$, respectively, we have
	\begin{equation*}
	\begin{aligned}
	P&=\frac{1}{n}\sum_{i=1}^{n-1}\nabla_\bold{w}\ell(\bold{w}_t,\bold{v}_t;z_i)+\frac{1}{n}\nabla_\bold{w}\ell(\bold{w}_t,\bold{v}_t;z_n)+b_\bold{w}, \\
	Q&=\frac{1}{n}\sum_{i=1}^{n-1}\nabla_\bold{w}\ell(\bold{w}_t,\bold{v}_t;z_i)+\frac{1}{n}\nabla_\bold{w}\ell(\bold{w}_t,\bold{v}_t;z_n')+b_\bold{w},
	\end{aligned}
	\end{equation*}
	where we assume the single different data sample is the $n^{th}$ one\footnote{This assumption is only for simplicity, the different data sample can be anyone in the dataset.}.
	
	Noting that $b_\bold{w}\sim\mathcal{N}(0,\sigma^2I_p)$, we have
	\begin{equation*}
	\begin{aligned}
	P&=\mathcal{N}\left(\frac{1}{n}\sum_{i=1}^{n-1}\nabla_\bold{w}\ell(\bold{w}_t,\bold{v}_t;z_i)+\frac{1}{n}\nabla_\bold{w}\ell(\bold{w}_t,\bold{v}_t;z_n),\sigma^2I_p\right), \\
	Q&=\mathcal{N}\left(\frac{1}{n}\sum_{i=1}^{n-1}\nabla_\bold{w}\ell(\bold{w}_t,\bold{v}_t;z_i)+\frac{1}{n}\nabla_\bold{w}\ell(\bold{w}_t,\bold{v}_t;z_n'),\sigma^2I_p\right).
	\end{aligned}
	\end{equation*}
	
	By Definition \ref{def6}, we have
	\begin{equation*}
	\alpha_{\mathcal{M}_t}(\lambda;D,D')=\log\left(\mathbb{E}_{o\sim P}\left[\exp\left(\lambda\log\left(\frac{P}{Q}\right)\right)\right]\right)=\log\left(\mathbb{E}_{o\sim P}\left[\left(\frac{P}{Q}\right)^\lambda\right]\right)=\lambda D_{\lambda+1}\left(P||Q\right),
	\end{equation*}
	where the last equality holds because of Definition \ref{def7}.
	
	Via Lemma \ref{lem12}, we have
	\begin{equation*}
	\alpha_{\mathcal{M}_t}(\lambda;D,D')=\frac{\lambda(\lambda+1)\left\Vert\frac{1}{n}\left(\nabla_\bold{w}\ell(\bold{w}_t,\bold{v}_t;z_n)-\nabla_\bold{w}\ell(\bold{w}_t,\bold{v}_t;z_n')\right)\right\Vert_2^2}{2\sigma^2}\leq\frac{2G^2\lambda(\lambda+1)}{n^2\sigma^2}=\alpha_{\mathcal{M}_t}(\lambda),
	\end{equation*}
	where the inequality holds because $\ell$ is $G$-Lipschitz and the last equality holds because of Definition \ref{def6}.
	
	Via Lemma \ref{lem10}, since there are $T$ iterations, we have
	\begin{equation*}
	\alpha_\mathcal{M}(\lambda)\leq\sum_{i=1}^{T}\alpha_{\mathcal{M}_t}(\lambda)\leq\frac{4G^2\lambda^2T}{n^2\sigma^2},
	\end{equation*}
	where the last inequality holds because $\lambda\in(1,\infty)$.
	
	Taking $\sigma=c\frac{G\sqrt{T\log(1/\delta)}}{n\epsilon}$, we can guarantee $\alpha_\mathcal{M}(\lambda)\leq\lambda\epsilon/2$ and as a result, we have $\delta\leq\exp(-\lambda\epsilon/2)$, which leads ($\epsilon,\delta$)-DP via Lemma \ref{lem11}.
	
	The training process over parameter $\bold{v}$ is similar, so if $b_{\bold{v}}\sim\mathcal{N}(0,\sigma^2I_p)$ with $\sigma=c\frac{G\sqrt{T\log(1/\delta)}}{n\epsilon}$ is injected into the gradient when updating $\bold{v}$, then ($\epsilon,\delta$)-DP can be guaranteed.
	
	Moreover, as discussed in Section 4, the average operator does not violate differential privacy because of the Post-Processing property \cite{dwork2014the}.
	
	The proof is complete.
	
\end{proof}

\subsection{A.2. Proof of Theorem \ref{the2}}

To get the stability bound, we further need the following lemma.

\begin{lemma}\label{lem1}[\cite{yang2021stability}]
	If Gaussian noise $b\sim\mathcal{N}(0,\sigma^2I_p)$, then for $\zeta\in(\exp(-\frac{p}{8}),1)$, with probability at least $1-\zeta$
	\begin{equation*}
	\|b\|_2\leq\sigma\sqrt{p}\left(1+\left(\frac{8\log(1/\zeta)}{p}\right)^{1/4}\right).
	\end{equation*}
\end{lemma}

Then, we start our proof.

\begin{theorem}
	If Assumptions \ref{a1} and \ref{a2} hold.
	Then with $\sigma$ given in Theorem \ref{the1} and $\eta_t=\frac{1}{\rho t}$, the output of DP-GDA (Algorithm 1) $\mathcal{A}(S)=(\bar{\bold{w}}_T,\bar{\bold{v}}_T)$ is $\gamma$-argument stable with probability at least $1-\zeta$ for $\zeta\in(\exp(-\frac{p}{8}),1)$, where
	\begin{equation*}
	\begin{aligned}
	\gamma&=\frac{4G}{n\rho}+2\sigma\sqrt{p}\log(eT)p_\zeta+4\sqrt{\log(eT)}\sqrt{\frac{G^2}{\rho^2T}+\frac{\sigma^2p}{\rho^2T}p_\zeta^2+\frac{2G\sigma\sqrt{p}}{\rho^2T}p_\zeta+\frac{\left(g_\bold{w}+g_\bold{v}\right)\sigma\sqrt{p}}{\rho\log(eT)}p_\zeta},
	\end{aligned}
	\end{equation*}
\end{theorem}
where $p_\zeta=1+\big(\frac{8\log(2T/\zeta)}{p}\big)^{1/4}$, $g_\bold{w}=\left\Vert\bar{\bold{w}}^*-\bar{\bold{w}}_{T}\right\Vert_2$, and $g_\bold{v}=\left\Vert\bar{\bold{v}}^*-\bar{\bold{v}}_{T}\right\Vert_2$ for $\bar{\bold{w}}^*=\arg\min_{\bold{w}\in\mathcal{W}}L_S(\bold{w},\bar{\bold{v}}_T)$, and $\bar{\bold{v}}^*=\arg\max_{\bold{v}\in\mathcal{V}}L_S(\bar{\bold{w}}_T,\bold{v})$.

\begin{proof}
	We define the dataset adjacent to $S$ as $S^{(i)}=\{z_1,\cdots,z_{i-1},z_i',z_{i+1},\cdots,z_n\}$.
	At iteration $t$, $(\bold{w}_t,\bold{v}_t)$ is the output derived from $S$ and $(\bold{w}_t^{(i)},\bold{v}_t^{(i)})$ derived from $S^{(i)}$.
	
	Let $\bold{w}_{S}^*,\bold{v}_{S}^*$ be the Empirical Saddle Point (ESP) over dataset $S$ and $\bold{w}_{S^{(i)}}^*,\bold{v}_{S^{(i)}}^*$ be the ESP over dataset $S^{(i)}$.
	Then we have
	\begin{equation*}
	\begin{aligned}
	&L_S\left(\bold{w}_{S^{(i)}}^*,\bold{v}_{S}^*\right)-L_S\left(\bold{w}_{S}^*,\bold{v}_{S^{(i)}}^*\right) \\
	&=\frac{1}{n}\sum_{j=1}^{n}\left(\ell(\bold{w}_{S^{(i)}}^*,\bold{v}_{S}^*;z_j)-\ell(\bold{w}_{S}^*,\bold{v}_{S^{(i)}}^*;z_j)\right) \\
	&=\frac{1}{n}\left(\sum_{j=1,j\neq i}^{n}\left(\ell(\bold{w}_{S^{(i)}}^*,\bold{v}_{S}^*;z_j)-\ell(\bold{w}_{S}^*,\bold{v}_{S^{(i)}}^*;z_j)\right)+\ell(\bold{w}_{S^{(i)}}^*,\bold{v}_{S}^*;z_i')-\ell(\bold{w}_{S}^*,\bold{v}_{S^{(i)}}^*;z_i')\right) \\
	&\quad+\frac{1}{n}\left(\ell(\bold{w}_{S^{(i)}}^*,\bold{v}_{S}^*;z_i)-\ell(\bold{w}_{S}^*,\bold{v}_{S^{(i)}}^*;z_i)\right)-\frac{1}{n}\left(\ell(\bold{w}_{S^{(i)}}^*,\bold{v}_{S}^*;z_i')-\ell(\bold{w}_{S}^*,\bold{v}_{S^{(i)}}^*;z_i')\right) \\
	&=L_{S^{(i)}}\left(\bold{w}_{S^{(i)}}^*,\bold{v}_{S}^*\right)-L_{S^{(i)}}\left(\bold{w}_{S}^*,\bold{v}_{S^{(i)}}^*\right) \\
	&\quad+\frac{1}{n}\left(\ell(\bold{w}_{S^{(i)}}^*,\bold{v}_{S}^*;z_i)-\ell(\bold{w}_{S^{(i)}}^*,\bold{v}_{S^{(i)}}^*;z_i)+\ell(\bold{w}_{S^{(i)}}^*,\bold{v}_{S^{(i)}}^*;z_i)-\ell(\bold{w}_{S}^*,\bold{v}_{S^{(i)}}^*;z_i)\right) \\
	&\quad-\frac{1}{n}\left(\ell(\bold{w}_{S^{(i)}}^*,\bold{v}_{S}^*;z_i')-\ell(\bold{w}_{S^{(i)}}^*,\bold{v}_{S^{(i)}}^*;z_i')+\ell(\bold{w}_{S^{(i)}}^*,\bold{v}_{S^{(i)}}^*;z_i')-\ell(\bold{w}_{S}^*,\bold{v}_{S^{(i)}}^*;z_i')\right) \\
	&\overset{(G)}{\leq}L_{S^{(i)}}\left(\bold{w}_{S^{(i)}}^*,\bold{v}_{S}^*\right)-L_{S^{(i)}}\left(\bold{w}_{S}^*,\bold{v}_{S^{(i)}}^*\right)+\frac{2G}{n}\left(\|\bold{w}_{S^{(i)}}^*-\bold{w}_{S}^*\|_2+\|\bold{v}_{S^{(i)}}^*-\bold{v}_{S}^*\|_2\right) \\
	&=L_{S^{(i)}}\left(\bold{w}_{S^{(i)}}^*,\bold{v}_{S}^*\right)-L_{S^{(i)}}\left(\bold{w}_{S^{(i)}}^*,\bold{v}_{S^{(i)}}^*\right)+L_{S^{(i)}}\left(\bold{w}_{S^{(i)}}^*,\bold{v}_{S^{(i)}}^*\right)-L_{S^{(i)}}\left(\bold{w}_{S}^*,\bold{v}_{S^{(i)}}^*\right) \\
	&\quad+\frac{2G}{n}\left(\|\bold{w}_{S^{(i)}}^*-\bold{w}_{S}^*\|_2+\|\bold{v}_{S^{(i)}}^*-\bold{v}_{S}^*\|_2\right) \\
	&\overset{(\rho)}{\leq}-\frac{\rho}{2}\left(\|\bold{w}_{S^{(i)}}^*-\bold{w}_{S}^*\|^2_2+\|\bold{v}_{S^{(i)}}^*-\bold{v}_{S}^*\|^2_2\right)+\frac{2G}{n}\left(\|\bold{w}_{S^{(i)}}^*-\bold{w}_{S}^*\|_2+\|\bold{v}_{S^{(i)}}^*-\bold{v}_{S}^*\|_2\right),
	\end{aligned}
	\end{equation*}
	where the first inequality holds because $\ell(\cdot,\cdot;\cdot)$ is $G$-Lipschitz, the second inequality holds because $L_{S^{(i)}}$ is $\rho$-SC-SC, $\nabla_\bold{w}L_{S^{(i)}}(\bold{w}_{S^{(i)}}^*,\cdot)=0$, and $\nabla_\bold{v}L_{S^{(i)}}(\cdot,\bold{v}_{S^{(i)}}^*)=0$.
	
	Similarly, since $L_{S}$ is $\rho$-SC-SC, we have
	\begin{equation}\label{eq1}
	\begin{aligned}
	L_{S}\left(\bold{w}_{S^{(i)}}^*,\bold{v}_{S}^*\right)-L_{S}\left(\bold{w}_{S}^*,\bold{v}_{S^{(i)}}^*\right)\geq\frac{\rho}{2}\left(\|\bold{w}_{S^{(i)}}^*-\bold{w}_{S}^*\|^2_2+\|\bold{v}_{S^{(i)}}^*-\bold{v}_{S}^*\|^2_2\right).
	\end{aligned}
	\end{equation}
	
	Thus
	\begin{equation*}
	\frac{\rho}{2}\left(\|\bold{w}_{S^{(i)}}^*-\bold{w}_{S}^*\|^2_2+\|\bold{v}_{S^{(i)}}^*-\bold{v}_{S}^*\|^2_2\right)\leq-\frac{\rho}{2}\left(\|\bold{w}_{S^{(i)}}^*-\bold{w}_{S}^*\|^2_2+\|\bold{v}_{S^{(i)}}^*-\bold{v}_{S}^*\|^2_2\right)+\frac{2G}{n}\left(\|\bold{w}_{S^{(i)}}^*-\bold{w}_{S}^*\|_2+\|\bold{v}_{S^{(i)}}^*-\bold{v}_{S}^*\|_2\right),
	\end{equation*}
	which derives
	\begin{equation*}
	\begin{aligned}
	\rho\left(\|\bold{w}_{S^{(i)}}^*-\bold{w}_{S}^*\|^2_2+\|\bold{v}_{S^{(i)}}^*-\bold{v}_{S}^*\|^2_2\right)&\leq\frac{2G}{n}\left(\|\bold{w}_{S^{(i)}}^*-\bold{w}_{S}^*\|_2+\|\bold{v}_{S^{(i)}}^*-\bold{v}_{S}^*\|_2\right) \\
	&\leq\frac{2\sqrt{2}G}{n}\sqrt{\|\bold{w}_{S^{(i)}}^*-\bold{w}_{S}^*\|_2^2+\|\bold{v}_{S^{(i)}}^*-\bold{v}_{S}^*\|_2^2},
	\end{aligned}
	\end{equation*}
	where the last inequality holds because $a+b\leq\sqrt{2(a^2+b^2)}$ for all $a,b>0$.
	
	Therefore, we have
	\begin{equation*}
	\sqrt{\|\bold{w}_{S^{(i)}}^*-\bold{w}_{S}^*\|_2^2+\|\bold{v}_{S^{(i)}}^*-\bold{v}_{S}^*\|_2^2}\leq\frac{2\sqrt{2}G}{n\rho}.
	\end{equation*}
	
	And
	\begin{equation}\label{eq2}
	\|\bold{w}_{S^{(i)}}^*-\bold{w}_{S}^*\|_2+\|\bold{v}_{S^{(i)}}^*-\bold{v}_{S}^*\|_2\leq\frac{4G}{n\rho}.
	\end{equation}
	
	Since the stability only depends on adjacent datasets, so we consider the noises injected to $S$ and $S^{(i)}$ are the same.
	As a result,
	\begin{equation*}
	\begin{aligned}
	&\|\bold{w}_t^{(i)}-\bold{w}_t\|_2+\|\bold{v}_t^{(i)}-\bold{v}_t\|_2 \\
	&=\|\bold{w}_t^{(i)}-\bold{w}_{S^{(i)}}^*+\bold{w}_{S^{(i)}}^*-\bold{w}_{S}^*+\bold{w}_{S}^*-\bold{w}_t\|_2+\|\bold{v}_t^{(i)}-\bold{v}_{S^{(i)}}^*+\bold{v}_{S^{(i)}}^*-\bold{v}_{S}^*+\bold{v}_{S}^*-\bold{v}_t\|_2 \\
	&\leq\|\bold{w}_t^{(i)}-\bold{w}_{S^{(i)}}^*\|_2+\|\bold{w}_{S^{(i)}}^*-\bold{w}_{S}^*\|_2+\|\bold{w}_{S}^*-\bold{w}_t\|_2+\|\bold{v}_t^{(i)}-\bold{v}_{S^{(i)}}^*\|_2+\|\bold{v}_{S^{(i)}}^*-\bold{v}_{S}^*\|_2+\|\bold{v}_{S}^*-\bold{v}_t\|_2 \\
	&\leq\frac{4G}{n\rho}+\|\bold{w}_t^{(i)}-\bold{w}_{S^{(i)}}^*\|_2+\|\bold{w}_{S}^*-\bold{w}_t\|_2+\|\bold{v}_t^{(i)}-\bold{v}_{S^{(i)}}^*\|_2+\|\bold{v}_{S}^*-\bold{v}_t\|_2 \\
	&\leq\frac{4G}{n\rho}+\sqrt{2\left(\|\bold{w}_t^{(i)}-\bold{w}_{S^{(i)}}^*\|_2^2+\|\bold{v}_t^{(i)}-\bold{v}_{S^{(i)}}^*\|_2^2\right)}+\sqrt{2\left(\|\bold{w}_{S}^*-\bold{w}_t\|_2^2+\|\bold{v}_{S}^*-\bold{v}_t\|_2^2\right)} \\
	&\leq\frac{4G}{n\rho}+\sqrt{\frac{4}{\rho}}\left(\sqrt{L_{S^{(i)}}(\bold{w}_t^{(i)},\bold{v}_{S^{(i)}}^*)-L_{S^{(i)}}(\bold{w}_{S^{(i)}}^*,\bold{v}_t^{(i)})}+\sqrt{L_S(\bold{w}_t,\bold{v}_{S}^*)-L_S(\bold{w}_{S}^*,\bold{v}_t)}\right),
	\end{aligned}
	\end{equation*}
	where the second inequality holds because of (\ref{eq2}), the third inequality holds because $a+b\leq\sqrt{2(a^2+b^2)}$ for all $a,b>0$, and the last inequality holds because both $L_S$ and $L_{S^{(i)}}$ are $\rho$-SC-SC (as discussed in (\ref{eq1})).
	
	Since terms $L_{S^{(i)}}(\bold{w}_t^{(i)},\bold{v}_{S^{(i)}}^*)-L_{S^{(i)}}(\bold{w}_{S^{(i)}}^*,\bold{v}_t^{(i)})$ and $L_S(\bold{w}_t,\bold{v}_{S}^*)-L_S(\bold{w}_{S}^*,\bold{v}_t)$ are both strong PD empirical risk, we bound them with the same measure for brevity:
	\begin{equation}\label{eq3}
	\|\bold{w}_t^{(i)}-\bold{w}_t\|_2+\|\bold{v}_t^{(i)}-\bold{v}_t\|_2\leq\frac{4G}{n\rho}+4\sqrt{\frac{1}{\rho}}\sqrt{\triangle_S^s(\bold{w}_t,\bold{v}_t)}.
	\end{equation}
	
	Till iteration $t$, noise is added to the model $t$ times on both $\bold{w}$ and $\bold{v}$, so when it comes to the noisy version, the fluctuations caused by the injected noise lead (\ref{eq3}) to:
	\begin{equation}\label{eq4}
	\|\bold{w}_t^{(i)}-\bold{w}_t\|_2+\|\bold{v}_t^{(i)}-\bold{v}_t\|_2\leq\frac{4G}{n\rho}+4\sqrt{\frac{1}{\rho}}\sqrt{\triangle_S^s(\bold{w}_t,\bold{v}_t)}+\sum_{k=1}^{t}\eta_k\|b_\bold{w}\|_2+\sum_{k=1}^{t}\eta_k\|b_\bold{v}\|_2.
	\end{equation}
	
	Now we bound the strong PD empirical risk $\triangle_S^s(\bold{w}_t,\bold{v}_t)$.
	
	Firstly, we have
	\begin{equation*}
	\begin{aligned}
	\|\bold{w}_{t+1}-\bold{w}\|_2^2&=\|\bold{w}_{t}-\eta_t\left(\nabla_\bold{w}L_S(\bold{w}_t,\bold{v}_t)+b_\bold{w}\right)-\bold{w}\|_2^2 \\
	&=\|\bold{w}_{t}-\bold{w}\|_2^2+\eta_t^2\|\nabla_\bold{w}L_S(\bold{w}_t,\bold{v}_t)+b_\bold{w}\|_2^2+2\eta_t\left\langle\bold{w}-\bold{w}_{t},\nabla_\bold{w}L_S(\bold{w}_t,\bold{v}_t)+b_\bold{w}\right\rangle \\
	&\leq\|\bold{w}_{t}-\bold{w}\|_2^2+\eta_t^2G^2+\eta_t^2\|b_\bold{w}\|_2^2+2\eta_t^2G\|b_\bold{w}\|_2+2\eta_t\left\langle\bold{w}-\bold{w}_{t},\nabla_\bold{w}L_S(\bold{w}_t,\bold{v}_t)\right\rangle+2\eta_t\left\langle\bold{w}-\bold{w}_{t},b_\bold{w}\right\rangle,
	\end{aligned}
	\end{equation*}
	where the last inequality holds because of Cauchy-Schwartz inequality.
	
	Note that $L_S(\cdot,\bold{v}_t)$ is $\rho$-strongly convex, we have
	\begin{equation*}
	L_S(\bold{w},\bold{v}_t)-L_S(\bold{w}_t,\bold{v}_t)\geq\left\langle\nabla_\bold{w}L_S(\bold{w}_t,\bold{v}_t),\bold{w}-\bold{w}_t\right\rangle+\frac{\rho}{2}\|\bold{w}-\bold{w}_t\|_2^2.
	\end{equation*}
	
	So
	\begin{equation*}
	\|\bold{w}_{t+1}-\bold{w}\|_2^2\leq\left(1-\eta_t\rho\right)\|\bold{w}-\bold{w}_t\|_2^2+\eta_t^2G^2+\eta_t^2\|b_\bold{w}\|_2^2+2\eta_t^2G\|b_\bold{w}\|_2+2\eta_t\left\langle\bold{w}-\bold{w}_{t},b_\bold{w}\right\rangle+2\eta_t\left(L_S(\bold{w},\bold{v}_t)-L_S(\bold{w}_t,\bold{v}_t)\right).
	\end{equation*}
	
	Taking $\eta_t=\frac{1}{\rho(t+\varphi)}$, we have
	\begin{equation*}
	\begin{aligned}
	\frac{2}{\rho(t+\varphi)}\left(L_S(\bold{w}_t,\bold{v}_t)-L_S(\bold{w},\bold{v}_t)\right)&\leq\left(1-\frac{1}{t+\varphi}\right)\|\bold{w}-\bold{w}_t\|_2^2-\|\bold{w}_{t+1}-\bold{w}\|_2^2+\frac{2}{\rho(t+\varphi)}\left\langle\bold{w}-\bold{w}_{t},b_\bold{w}\right\rangle \\
	&\quad+\left(\frac{1}{\rho(t+\varphi)}\right)^2G^2+\left(\frac{1}{\rho(t+\varphi)}\right)^2\|b_\bold{w}\|_2^2+2\left(\frac{1}{\rho(t+\varphi)}\right)^2G\|b_\bold{w}\|_2.
	\end{aligned}
	\end{equation*}
	
	Multiplying both sides by $t+\varphi$, we have
	\begin{equation*}
	\begin{aligned}
	\frac{2}{\rho}\left(L_S(\bold{w}_t,\bold{v}_t)-L_S(\bold{w},\bold{v}_t)\right)&\leq\left(t+\varphi-1\right)\|\bold{w}-\bold{w}_t\|_2^2-\left(t+\varphi\right)\|\bold{w}_{t+1}-\bold{w}\|_2^2+\frac{2}{\rho}\left\langle\bold{w}-\bold{w}_{t},b_\bold{w}\right\rangle \\
	&\quad+\frac{G^2}{\rho^2(t+\varphi)}+\frac{1}{\rho^2(t+\varphi)}\|b_\bold{w}\|_2^2+\frac{2G}{\rho^2(t+\varphi)}\|b_\bold{w}\|_2.
	\end{aligned}
	\end{equation*}
	
	Since $\sum_{t=1}^T\frac{1}{t}\leq\log(eT)$, by summing over $T$ iterations, the following inequality holds
	\begin{equation*}
	\begin{aligned}
	\sum_{t=1}^T\left(L_S(\bold{w}_t,\bold{v}_t)-L_S(\bold{w},\bold{v}_t)\right)&\leq\frac{\rho\varphi}{2}\|\bold{w}-\bold{w}_1\|_2^2+\sum_{t=1}^T\left\langle\bold{w}-\bold{w}_{t},b_\bold{w}\right\rangle+\frac{G^2\log(eT)}{2\rho}+\frac{\log(eT)}{2\rho}\|b_\bold{w}\|_2^2+\frac{G\log(eT)}{\rho}\|b_\bold{w}\|_2.
	\end{aligned}
	\end{equation*}
	
	For bounded $\mathcal{W}$ (i.e. $\|\bold{w}-\bold{w}'\|_2\leq M_\mathcal{W}$ for all $\bold{w},\bold{w}'$), we have:
	\begin{equation*}
	\sum_{t=1}^T\left(L_S(\bold{w}_t,\bold{v}_t)-L_S(\bold{w},\bold{v}_t)\right)\leq\frac{\rho\varphi}{2}M_\mathcal{W}^2+\sum_{t=1}^{T}\left\langle\bold{w}-\bold{w}_{t},b_\bold{w}\right\rangle+\frac{G^2\log(eT)}{2\rho}+\frac{\log(eT)}{2\rho}\|b_\bold{w}\|_2^2+\frac{G\log(eT)}{\rho}\|b_\bold{w}\|_2.
	\end{equation*}
	
	With the concavity of $L_S(\bold{w},\cdot)$, we have $TL_S(\bar{\bold{v}}_T)\geq \sum_{t=1}^TL_S(\bold{v}_t)$, and noting that the inequality holds for any $\bold{w}$, thus for $\bar{\bold{w}}^*=\arg\min_{\bold{w}\in\mathcal{W}}L_S(\bold{w},\bar{\bold{v}}_T)$
	\begin{equation*}
	\sum_{t=1}^TL_S(\bold{w}_t,\bold{v}_t)-T\inf_{\bold{w}\in\mathcal{W}}L_S(\bold{w},\bar{\bold{v}}_T)\leq\frac{\rho\varphi}{2}M_\mathcal{W}^2+\sum_{t=1}^{T}\left\langle\bar{\bold{w}}^*-\bold{w}_{t},b_\bold{w}\right\rangle+\frac{G^2\log(eT)}{2\rho}+\frac{\log(eT)}{2\rho}\|b_\bold{w}\|_2^2+\frac{G\log(eT)}{\rho}\|b_\bold{w}\|_2,
	\end{equation*}
	which implies
	\begin{equation}\label{eq5}
	\frac{1}{T}\sum_{t=1}^TL_S(\bold{w}_t,\bold{v}_t)-\inf_{\bold{w}\in\mathcal{W}}L_S(\bold{w},\bar{\bold{v}}_T)\leq\frac{\rho\varphi}{2T}M_\mathcal{W}^2+\left\langle\bar{\bold{w}}^*-\bar{\bold{w}}_{T},b_\bold{w}\right\rangle+\frac{G^2\log(eT)}{2\rho T}+\frac{\log(eT)}{2\rho T}\|b_\bold{w}\|_2^2+\frac{G\log(eT)}{\rho T}\|b_\bold{w}\|_2.
	\end{equation}
	
	Similarly, with bounded $\mathcal{V}$ (i.e. $\|\bold{v}-\bold{v}'\|_2\leq M_\mathcal{V}$ for all $\bold{v},\bold{v}'$), and denoting $\bar{\bold{v}}^*=\arg\max_{\bold{v}\in\mathcal{V}}L_S(\bar{\bold{w}}_T,\bold{v})$, we have
	\begin{equation}\label{eq6}
	\sup_{\bold{v}\in\mathcal{V}}L_S(\bar{\bold{w}}_T,\bold{v})-\frac{1}{T}\sum_{t=1}^TL_S(\bold{w}_t,\bold{v}_t)\leq\frac{\rho\varphi}{2T}M_\mathcal{V}^2+\left\langle\bar{\bold{v}}^*-\bar{\bold{v}}_{T},b_\bold{v}\right\rangle+\frac{G^2\log(eT)}{2\rho T}+\frac{\log(eT)}{2\rho T}\|b_\bold{v}\|_2^2+\frac{G\log(eT)}{\rho T}\|b_\bold{v}\|_2.
	\end{equation}
	
	Combining inequalities (\ref{eq5}) and (\ref{eq6}) together, with Cauchy-Schwartz inequality, we have
	\begin{equation}\label{eq33}
	\begin{aligned}
	&\triangle_S^s(\bar{\bold{w}}_T,\bar{\bold{v }}_T) \\
	&=\sup_{\bold{v}\in\mathcal{V}}L_S(\bar{\bold{w}}_T,\bold{v})-\inf_{\bold{w}\in\mathcal{W}}L_S(\bold{w},\bar{\bold{v}}_T) \\
	&\leq\frac{\rho\varphi}{2T}\left(M_\mathcal{W}^2+M_\mathcal{V}^2\right)+\frac{G^2\log(eT)}{\rho T}+\frac{\log(eT)}{2\rho T}\left(\|b_\bold{w}\|_2^2+\|b_\bold{v}\|_2^2\right)+\frac{G\log(eT)}{\rho T}\left(\|b_\bold{w}\|_2+\|b_\bold{v}\|_2\right)+g_\bold{w}\|b_\bold{w}\|_2+g_\bold{v}\|b_\bold{v}\|_2,
	\end{aligned}
	\end{equation}
	where $g_\bold{w}=\left\Vert\bar{\bold{w}}^*-\bar{\bold{w}}_{T}\right\Vert_2$ and $g_\bold{v}=\left\Vert\bar{\bold{v}}^*-\bar{\bold{v}}_{T}\right\Vert_2$.
	
	Noting that this empirical risk bound holds for all datasets (no matter $S$ or $S^{(i)}$), so (\ref{eq3}) holds.
	
	Taking this result back to (\ref{eq4}), if taking $\varphi=0$, for the average iterates, we have
	\begin{equation}\label{eq34}
	\begin{aligned}
	&\|\bar{\bold{w}}_T^{(i)}-\bar{\bold{w}}_T\|_2+\|\bar{\bold{v}}_T^{(i)}-\bar{\bold{v}}_T\|_2 \\
	&\leq\frac{4G}{n\rho}+4\sqrt{\frac{1}{\rho}}\sqrt{\triangle_S^s(\bar{\bold{w}}_T,\bar{\bold{v}}_T)}+\frac{1}{T}\sum_{t=1}^{T}\eta_t\left(\|b_\bold{w}\|_2+\|b_\bold{v}\|_2\right) \\
	&\leq\frac{4G}{n\rho}+\frac{\log(eT)}{T}\left(\|b_\bold{w}\|_2+\|b_\bold{v}\|_2\right) \\
	&\quad+4\sqrt{\frac{G^2\log(eT)}{\rho^2T}+\frac{\log(eT)}{2\rho^2T}\left(\|b_\bold{w}\|_2^2+\|b_\bold{v}\|_2^2\right)+\frac{G\log(eT)}{\rho^2T}\left(\|b_\bold{w}\|_2+\|b_\bold{v}\|_2\right)+\frac{g_\bold{w}\|b_\bold{w}\|_2+g_\bold{v}\|b_\bold{v}\|_2}{\rho}}.
	\end{aligned}
	\end{equation}
	
	So, GDA is $\gamma$-argument stability where
	\begin{equation}\label{eq7}
	\begin{aligned}
	\gamma&=4\sqrt{\frac{G^2\log(eT)}{\rho^2T}+\frac{\log(eT)}{2\rho^2T}\left(\|b_\bold{w}\|_2^2+\|b_\bold{v}\|_2^2\right)+\frac{G\log(eT)}{\rho^2T}\left(\|b_\bold{w}\|_2+\|b_\bold{v}\|_2\right)+\frac{g_\bold{w}\|b_\bold{w}\|_2+g_\bold{v}\|b_\bold{v}\|_2}{\rho}} \\
	&\quad+\frac{4G}{n\rho}+\frac{\log(eT)}{T}\left(\|b_\bold{w}\|_2+\|b_\bold{v}\|_2\right).
	\end{aligned}
	\end{equation}
	
	Noting that $b_\bold{w},b_\bold{v}$ mentioned in (\ref{eq7}) are derived from $T$ iterations, then via Lemma \ref{lem1}, for $\zeta\in(\exp(-\frac{p}{8}),1)$ and $p_\zeta=1+\left(\frac{8\log(2T/\zeta)}{p}\right)^{1/4}$, with probability at least $1-\zeta$, we have
	\begin{equation}\label{eq35}
	\begin{aligned}
	\gamma\leq4\sqrt{\frac{G^2\log(eT)}{\rho^2T}+\frac{\sigma^2p\log(eT)}{\rho^2T}p_\zeta^2+\frac{2G\sigma\sqrt{p}\log(eT)}{\rho^2T}p_\zeta+\frac{\left(g_\bold{w}+g_\bold{v}\right)\sigma\sqrt{p}}{\rho}p_\zeta}+\frac{4G}{n\rho}+\frac{2\sigma\sqrt{p}\log(eT)}{T}p_\zeta,
	\end{aligned}
	\end{equation}
	where the $\log(\cdot)$ term in $p_\zeta$ depends on probability over $T$ iterations and parameters $\bold{w},\bold{v}$, which completes the proof.
	
\end{proof}

\subsection{A.3. Proof of Theorem \ref{the3}}
Before the detailed proof, we first introduce the following lemmas.

\begin{lemma}\label{lem2}[\cite{bousquet2020sharper}]
	Let $S=\{z_1,\cdots,z_n\}$ be a set of independent random variables each taking values in $\mathcal{Z}$ and $M>0$. Define $S\setminus\{z_i\}=\{z_1,\cdots,z_{i-1},z_{i+1},\cdots,z_n\}$. Let $g_1,\cdots,g_n$ be some functions $g_i:\mathcal{Z}^n\rightarrow\mathbb{R}$ such that the following inequalities hold for any $i\in[1,n]$:
	\begin{itemize}
		\item $\mathbb{E}_{S\setminus\{z_i\}}[g_i(S)]\leq M$ almost surely;
		\item $\mathbb{E}_{z_i}[g_i(S)]=0$ almost surely;
		\item for any $j=1,\cdots,n$ ($j\neq i$) and $z_j'\in\mathcal{Z}$
		\begin{equation*}
		g_i(S)-g_i(z_1,\cdots,z_{j-1},z_j',z_{j+1},\cdots,n)\leq\beta.
		\end{equation*}
	\end{itemize}
	Then for any $\tau\geq2$
	\begin{equation*}
	\left\Vert\sum_{i=1}^ng_i(S)\right\Vert_\tau\leq12\sqrt{2}\tau n\beta\lceil\log(n)\rceil+4M\sqrt{\tau n}.
	\end{equation*}
\end{lemma}

\begin{lemma}\label{lem3}[\cite{bousquet2020sharper}]
	For some $a,b>0$ and any $\tau\geq2$, $X$ is a random variable satisfies
	\begin{equation*}
	\|X\|_\tau\leq\sqrt{\tau}a+\tau b.
	\end{equation*}
	Then for any $\zeta\in(0,1)$, with probability at least $1-\zeta$
	\begin{equation*}
	|X|\leq e\left(a\sqrt{\log\left(\frac{e}{\zeta}\right)}+b\log\left(\frac{e}{\zeta}\right)\right).
	\end{equation*}
\end{lemma}

\begin{lemma}\label{lem4}[\cite{boucheron2013concentration}]
	Let $z_1,\cdots,z_n$ be i.i.d random variables and assume that $\mathbb{E}[z_i]=\mu$. Suppose $|z_i|<c$ for any $i$. Then for any $\zeta\in(0,1)$, with probability at least $1-\zeta$
	\begin{equation*}
	\left|\frac{1}{n}\sum_{i=1}^{n}z_i-\mu\right|\leq\sqrt{\frac{2\sigma^2\log(1/\zeta)}{n}}+\frac{2c\log(1/\zeta)}{3n},
	\end{equation*}
	where $\sigma^2$ is the variance of $z_i$.
\end{lemma}

\begin{lemma}\label{lem5}[\cite{klochkov2021stability}]
	If function $f:\mathcal{Z}\rightarrow[0,+\infty)$ is ($a,b$)-weakly self-bounded, corresponding $f_i(Z^n)\geq f(Z^n)$ for $i=1,\cdots,n$ and any $Z^n\in\mathcal{Z}^{n}$, and $z_1,\cdots,z_n$ are independent random variables. Then, for any $t>0$
	\begin{equation*}
	\mathbb{P}\left(\mathbb{E}\left[f(z_1,\cdots,z_n)\right]\geq f(z_1,\cdots,z_n)+t\right)\leq\exp\left(-\frac{t^2}{2a\mathbb{E}\left[f(z_1,\cdots,z_n)\right]+2b}\right).
	\end{equation*}
\end{lemma}

\begin{lemma}\label{lem6}[\cite{zhang2021generalization}]
	Assume $f:\mathcal{W}\times\mathcal{V}\rightarrow\mathbb{R}$ is $\rho$-strongly-convex-strongly-concave and $L$-smooth.
	Defining $\bold{w}^*(\bold{v})=\arg\min_{\bold{w}\in\mathcal{W}}f(\bold{w},\bold{v})$ for any $\bold{v}$, and $\bold{v}^*(\bold{w})=\arg\max_{\bold{v}\in\mathcal{V}}f(\bold{w},\bold{v})$ for any $\bold{w}$. Then for any $\bold{w},\bold{w}'\in\mathcal{W}$ and $\bold{v},\bold{v}'\in\mathcal{V}$ there holds that
	\begin{equation*}
	\|\bold{w}^*(\bold{v})-\bold{w}^*(\bold{v}')\|_2\leq\frac{L}{\rho}\|\bold{v}-\bold{v}'\|_2,\quad\|\bold{v}^*(\bold{w})-\bold{v}^*(\bold{w}')\|_2\leq\frac{L}{\rho}\|\bold{w}-\bold{w}'\|_2.
	\end{equation*}
\end{lemma}

\begin{lemma}\label{lem7}[\cite{boucheron2013concentration}]
	If $g_1,\cdots,g_n$ are i.i.d, zero mean and $|g_i|\leq M$ almost surely, then for any $\tau\geq2$,
	\begin{equation*}
	\left\Vert\sum_{i=1}^{n}g_i\right\Vert_\tau\leq6\sqrt{\left(\sum_{i=1}^{n}\mathbb{E}\left[g_i^2\right]\right)\tau}+4\tau M.
	\end{equation*}
\end{lemma}

Besides, we need the following definition.
\begin{definition}[Weakly Self-Bounded Function]\label{def3}
	For all $Z^n\in\mathcal{Z}^{n}$, function $f:\mathcal{Z}\rightarrow[0,+\infty)$ is ($a,b$)-weakly self-bounded $(a,b>0)$ if there exists $f_i:\mathcal{Z}^{n-1}\rightarrow[0,+\infty)$ that satisfies
	\begin{equation*}
	\sum_{i=1}^{n}\left(f(Z^n)-f_i(Z^n)\right)^2\leq af(Z^n)+b.
	\end{equation*}
\end{definition}

Then, we recall Theorem \ref{the3} and start our proof.

\begin{theorem} With argument stability parameter $\gamma$,
	
	(a) If Assumptions \ref{a2} and \ref{a4} hold, then for all $\iota,\zeta>0$, with probability at least $1-\zeta$, the plain generalization error satisfies
	\begin{equation*}
	\begin{aligned}
	&L(\mathcal{A}_\bold{w}(S),\mathcal{A}_\bold{v}(S))-\frac{1}{1-\iota}L_S(\mathcal{A}_\bold{w}(S),\mathcal{A}_\bold{v}(S)) \\
	&\leq\sqrt{\frac{\left(G^2\gamma^2+64G^2n\gamma^2\log\left(3/\zeta\right)\right)}{2\left(1-\iota\right)^2n}\log\left(\frac{3}{\zeta}\right)}+\frac{50\sqrt{2}eG\gamma\log(n)}{1-\iota}\log\left(\frac{3e}{\zeta}\right)+\frac{\left(12+2\iota\right)M_\ell}{3\iota\left(1-\iota\right)n}\log\left(\frac{3}{\zeta}\right).
	\end{aligned}
	\end{equation*}
	
	(b) If Assumptions \ref{a1}, \ref{a2}, \ref{a3}, \ref{a4} and \ref{a5} hold, then for all $\iota,\zeta>0$, with probability at least $1-\zeta$, the primal generalizaiton error staisfies
	\begin{equation*}
	\begin{aligned}
	&R\left(\mathcal{A}_\bold{w}(S)\right)-\frac{1}{1-\iota}R_S\left(\mathcal{A}_\bold{w}(S)\right) \\
	&\leq\sqrt{\frac{\left(1+L/\rho\right)^2G^2\gamma^2\left(1+64n\log\left(3/\zeta\right)\right)}{2(1-\iota)^2n}\log\left(\frac{3}{\zeta}\right)}+\frac{50\sqrt{2}\left(1+L/\rho\right)G\gamma\log(n)}{1-\iota}\log\left(\frac{3e}{\zeta}\right)+\frac{(12+2\iota)M_\ell}{3\iota(1-\iota)n}\log\left(\frac{3}{\zeta}\right).
	\end{aligned}
	\end{equation*}
	
	(c) If Assumptions \ref{a1}, \ref{a2}, \ref{a3}, \ref{a4} and \ref{a5} hold, then for all $\iota,\zeta>0$, with probability at least $1-\zeta$, the primal excess population risk satisfies
	\begin{equation*}
	\begin{aligned}
	&R\left(\mathcal{A}_\bold{w}(S)\right)-\frac{1+\iota}{1-\iota}\inf_{\bold{w}\in\mathcal{W}}R\left(\bold{w}\right) \\
	&\leq\sqrt{\frac{\left(1+L/\rho\right)^2G^2\gamma^2\left(1+64n\log\left(6/\zeta\right)\right)}{2(1-\iota)^2n}\log\left(\frac{6}{\zeta}\right)}+\sqrt{\frac{\left(G^2\gamma^2+64G^2n\gamma^2\log\left(6/\zeta\right)\right)}{2(1-\iota)^2n}\log\left(\frac{6}{\zeta}\right)} \\
	&\quad+\frac{50\sqrt{2}\left(1+e+L/\rho\right)G\gamma\log(n)}{1-\iota}\log\left(\frac{6e}{\zeta}\right)+\frac{(24+4\iota)M_\ell}{3\iota(1-\iota)n}\log\left(\frac{6}{\zeta}\right)+\frac{1}{1-\iota}\triangle_S^s\left(\mathcal{A}_\bold{w}(S),\mathcal{A}_\bold{v}(S)\right).
	\end{aligned}
	\end{equation*}
	
	(d) If Assumptions \ref{a1}, \ref{a2}, \ref{a3}, \ref{a4} and \ref{a5} hold, then for all $\iota,\zeta>0$, with probability at least $1-\zeta$, the strong primal dual population risk satisfies
	\begin{equation*}
	\begin{aligned}
	&\triangle^s\left(\mathcal{A}_\bold{w}(S),\mathcal{A}_\bold{v}(S)\right) \\
	&\leq\frac{100\sqrt{2}e(1+\iota)(1+L/\rho)G\gamma\log(n)}{1-\iota}\log\left(\frac{e}{\zeta}\right)+\frac{144e(1+\iota)G^2}{\rho\iota(1-\iota)n}\log\left(\frac{e}{\zeta}\right)+\frac{8e(1+\iota)M_\ell}{n(1-\iota)}\log\left(\frac{e}{\zeta}\right) \\
	&\quad+\left(\frac{e\iota}{1-\iota}\log\left(\frac{e}{\zeta}\right)+1\right)\triangle_S^s\left(\mathcal{A}_\bold{w}(S),\mathcal{A}_\bold{v}(S)\right).
	\end{aligned}
	\end{equation*}
\end{theorem}

\begin{proof}
	\textbf{Part (a): The plain generalization error}.
	
	Fristly , we have
	\begin{equation*}
	\begin{aligned}
	nL(\mathcal{A}_\bold{w}(S),\mathcal{A}_\bold{v}(S))-nL_S(\mathcal{A}_\bold{w}(S),\mathcal{A}_\bold{v}(S))&=\sum_{i=1}^{n}\mathbb{E}_z\left[\ell(\mathcal{A}_\bold{w}(S),\mathcal{A}_\bold{v}(S);z)-\mathbb{E}_{z_i'}\left[\ell(\mathcal{A}_\bold{w}(S^{(i)}),\mathcal{A}_\bold{v}(S^{(i)});z)\right]\right] \\
	&\quad+\sum_{i=1}^{n}\mathbb{E}_{z_i'}\left[\mathbb{E}_z\left[\ell(\mathcal{A}_\bold{w}(S^{(i)}),\mathcal{A}_\bold{v}(S^{(i)});z)\right]-\ell(\mathcal{A}_\bold{w}(S^{(i)}),\mathcal{A}_\bold{v}(S^{(i)});z_i)\right] \\
	&\quad+\sum_{i=1}^{n}\mathbb{E}_{z_i'}\left[\ell(\mathcal{A}_\bold{w}(S^{(i)}),\mathcal{A}_\bold{v}(S^{(i)});z_i)\right]-\sum_{i=1}^{n}\ell(\mathcal{A}_\bold{w}(S),\mathcal{A}_\bold{v}(S);z_i).
	\end{aligned}
	\end{equation*}
	
	If algorithm $\mathcal{A}$ is $\gamma$-argument-stable, we have
	\begin{equation}\label{eq32}
	\begin{aligned}
	&\ell\left(\mathcal{A}_\bold{w}(S),\mathcal{A}_\bold{v}(S);z\right)-\ell\left(\mathcal{A}_\bold{w}(S'),\mathcal{A}_\bold{v}(S');z\right) \\
	&=\ell\left(\mathcal{A}_\bold{w}(S),\mathcal{A}_\bold{v}(S);z\right)-\ell\left(\mathcal{A}_\bold{w}(S'),\mathcal{A}_\bold{v}(S);z\right)+\ell\left(\mathcal{A}_\bold{w}(S'),\mathcal{A}_\bold{v}(S);z\right)-\ell\left(\mathcal{A}_\bold{w}(S'),\mathcal{A}_\bold{v}(S');z\right) \\
	&\leq G\left(\left\Vert\mathcal{A}_\bold{w}(S)-\mathcal{A}_\bold{w}(S')\right\Vert_2+\left\Vert\mathcal{A}_\bold{v}(S)-\mathcal{A}_\bold{v}(S')\right\Vert_2\right) \\
	&\leq G\gamma,
	\end{aligned}
	\end{equation}
	where the first inequality holds because of the $G$-Lipschitz property and the last inequality holds because of the definition of $\gamma$-argument stability.
	
	Defining $p_i(S)=\mathbb{E}_{z_i'}\left[\mathbb{E}_z\left[\ell(\mathcal{A}_\bold{w}(S^{(i)}),\mathcal{A}_\bold{v}(S^{(i)});z)\right]-\ell(\mathcal{A}_\bold{w}(S^{(i)}),\mathcal{A}_\bold{v}(S^{(i)});z_i)\right]$, then
	\begin{equation}\label{eq8}
	nL(\mathcal{A}_\bold{w}(S),\mathcal{A}_\bold{v}(S))-nL_S(\mathcal{A}_\bold{w}(S),\mathcal{A}_\bold{v}(S))\leq2Gn\gamma+\sum_{i=1}^{n}p_i(S).
	\end{equation}
	
	Furthermore, we define $q_i(S)=p_i(S)-\mathbb{E}_{S\setminus\{z_i\}}[p_i(S)]$.
	It is easy to follow that $\mathbb{E}_{S\setminus\{z_i\}}[q_i(S)]=0$ and $\mathbb{E}_{z_i}[q_i(S)]=\mathbb{E}_{z_i}[p_i(S)]-\mathbb{E}_{z_i}\mathbb{E}_{S\setminus\{z_i\}}[p_i(S)]=0$.
	
	For any $j=1,\cdots,n$ ($j\neq i$) and $z_j'\in\mathcal{Z}$, we have
	\begin{equation*}
	\begin{aligned}
	q_i(S)-q_i(z_1,\cdots,z_{j-1},z_j',z_{j+1},\cdots,n)&=p_i(S)-p_i(z_1,\cdots,z_{j-1},z_j',z_{j+1},\cdots,n) \\
	&\quad+\mathbb{E}_{S\setminus\{z_i\}}[p_i(z_1,\cdots,z_{j-1},z_j',z_{j+1},\cdots,n)]-\mathbb{E}_{S\setminus\{z_i\}}[p_i(S)] \\
	&\leq4G\gamma,
	\end{aligned}
	\end{equation*}
	where the last inequality holds because of (\ref{eq32}).
	
	Then, via Lemma \ref{lem2}, we have
	\begin{equation}\label{eq9}
	\left\Vert\sum_{i=1}^nq_i(S)\right\Vert_\tau\leq48\sqrt{2}G\tau n\gamma\lceil\log(n)\rceil.
	\end{equation}
	
	By the definition of $p_i(S)$ and $q_i(S)$, we have
	\begin{equation}\label{eq10}
	\begin{aligned}
	\sum_{i=1}^{n}\left(q_i(S)-p_i(S)\right)&=-\sum_{i=1}^{n}\mathbb{E}_{S\setminus\{z_i\}}[p_i(S)] \\
	&=-\left(n\mathbb{E}_{S'}L(\mathcal{A}_\bold{w}(S'),\mathcal{A}_\bold{v}(S'))-n\mathbb{E}_{S'}L_S(\mathcal{A}_\bold{w}(S'),\mathcal{A}_\bold{v}(S'))\right) \\
	&=-\left(n\mathbb{E}_{S}L(\mathcal{A}_\bold{w}(S),\mathcal{A}_\bold{v}(S))-n\mathbb{E}_{S'}L_S(\mathcal{A}_\bold{w}(S'),\mathcal{A}_\bold{v}(S'))\right),
	\end{aligned}
	\end{equation}
	where the last equality holds because of the i.i.d property, i.e. $\mathbb{E}_{S'}L(\mathcal{A}_\bold{w}(S'),\mathcal{A}_\bold{v}(S'))=\mathbb{E}_{S}L(\mathcal{A}_\bold{w}(S),\mathcal{A}_\bold{v}(S))$.
	
	Combining (\ref{eq8}), (\ref{eq9}), and (\ref{eq10}) together, for $\tau\geq2$, we have
	\begin{equation*}
	\begin{aligned}
	&\left\Vert nL(\mathcal{A}_\bold{w}(S),\mathcal{A}_\bold{v}(S))-nL_S(\mathcal{A}_\bold{w}(S),\mathcal{A}_\bold{v}(S))-\left(n\mathbb{E}_{S}L(\mathcal{A}_\bold{w}(S),\mathcal{A}_\bold{v}(S))-n\mathbb{E}_{S'}L_S(\mathcal{A}_\bold{w}(S'),\mathcal{A}_\bold{v}(S'))\right)\right\Vert_\tau \\
	&\leq\left\Vert nL(\mathcal{A}_\bold{w}(S),\mathcal{A}_\bold{v}(S))-nL_S(\mathcal{A}_\bold{w}(S),\mathcal{A}_\bold{v}(S))-\sum_{i=1}^{n}p_i(S)\right\Vert_\tau \\
	&\quad+\left\Vert\sum_{i=1}^{n}p_i(S)-\left(n\mathbb{E}_{S}L(\mathcal{A}_\bold{w}(S),\mathcal{A}_\bold{v}(S))-n\mathbb{E}_{S'}L_S(\mathcal{A}_\bold{w}(S'),\mathcal{A}_\bold{v}(S'))\right)\right\Vert_\tau \\
	&=\left\Vert nL(\mathcal{A}_\bold{w}(S),\mathcal{A}_\bold{v}(S))-nL_S(\mathcal{A}_\bold{w}(S),\mathcal{A}_\bold{v}(S))-\sum_{i=1}^{n}p_i(S)\right\Vert_\tau+\left\Vert\sum_{i=1}^{n}q_i(S)\right\Vert_\tau \\
	&\leq2Gn\gamma+48\sqrt{2}\tau n\gamma\lceil\log(n)\rceil \\
	&\leq50\sqrt{2}G\tau n\gamma\lceil\log(n)\rceil.
	\end{aligned}
	\end{equation*}
	
	Via Lemma \ref{lem3}, with probability at least $1-\zeta$, we have
	\begin{equation}\label{eq11}
	\begin{aligned}
	&L(\mathcal{A}_\bold{w}(S),\mathcal{A}_\bold{v}(S))-L_S(\mathcal{A}_\bold{w}(S),\mathcal{A}_\bold{v}(S)) \\
	&\leq\left|\mathbb{E}_{S'}L_S(\mathcal{A}_\bold{w}(S'),\mathcal{A}_\bold{v}(S'))-\mathbb{E}_{S}L(\mathcal{A}_\bold{w}(S),\mathcal{A}_\bold{v}(S))\right|+50\sqrt{2}eG\gamma\lceil\log(n)\rceil\log\left(\frac{e}{\zeta}\right).
	\end{aligned}
	\end{equation}
	
	Now we bound term $\mathbb{E}_{S'}L_S(\mathcal{A}_\bold{w}(S'),\mathcal{A}_\bold{v}(S'))-\mathbb{E}_{S}L(\mathcal{A}_\bold{w}(S),\mathcal{A}_\bold{v}(S))$.
	First, we consider $\mathbb{E}_{S'}\ell(\mathcal{A}_\bold{w}(S'),\mathcal{A}_\bold{v}(S');z)$.
	
	By Jensen's inequality, we have
	\begin{equation*}
	\begin{aligned}
	\mathbb{E}_{z_i}\left[\left(\mathbb{E}_{S'}\ell(\mathcal{A}_\bold{w}(S'),\mathcal{A}_\bold{v}(S');z_i)\right)^2\right]&\leq\mathbb{E}_{z_i}\left[\mathbb{E}_{S'}\left[\left(\ell(\mathcal{A}_\bold{w}(S'),\mathcal{A}_\bold{v}(S');z_i)\right)^2\right]\right] \\
	&=\mathbb{E}_{z}\left[\mathbb{E}_{S'}\left[\left(\ell(\mathcal{A}_\bold{w}(S'),\mathcal{A}_\bold{v}(S');z)\right)^2\right]\right] \\
	&=\mathbb{E}_{z}\left[\mathbb{E}_{S}\left[\left(\ell(\mathcal{A}_\bold{w}(S),\mathcal{A}_\bold{v}(S);z)\right)^2\right]\right].
	\end{aligned}
	\end{equation*}
	
	Noting that  $\mathbb{E}\left[\mathbb{E}_{z_i}\left[\mathbb{E}_{S'}\ell(\mathcal{A}_\bold{w}(S'),\mathcal{A}_\bold{v}(S');z_i)\right]\right]=\mathbb{E}_{S}L(\mathcal{A}_\bold{w}(S),\mathcal{A}_\bold{v}(S))$ and via Lemma \ref{lem4}, if $\ell(\cdot,\cdot;\cdot)$ is bounded by $M_\ell$, we have
	\begin{equation}\label{eq12}
	\left|\mathbb{E}_{S'}L_S(\mathcal{A}_\bold{w}(S'),\mathcal{A}_\bold{v}(S'))-\mathbb{E}_{S}L(\mathcal{A}_\bold{w}(S),\mathcal{A}_\bold{v}(S))\right|\leq\sqrt{\frac{2\mathbb{E}_{z}\left[\mathbb{E}_{S}\left[\left(\ell(\mathcal{A}_\bold{w}(S),\mathcal{A}_\bold{v}(S);z)\right)^2\right]\right]\log(1/\zeta)}{n}}+\frac{2M_\ell\log(1/\zeta)}{3n}.
	\end{equation}
	
	Combining (\ref{eq11}) and (\ref{eq12}) together, with probability at least $1-\zeta$, we have
	\begin{equation}\label{eq13}
	\begin{aligned}
	&L(\mathcal{A}_\bold{w}(S),\mathcal{A}_\bold{v}(S))-L_S(\mathcal{A}_\bold{w}(S),\mathcal{A}_\bold{v}(S)) \\
	&\leq\sqrt{\frac{2\mathbb{E}_{z}\left[\mathbb{E}_{S}\left[\left(\ell(\mathcal{A}_\bold{w}(S),\mathcal{A}_\bold{v}(S);z)\right)^2\right]\right]\log(2/\zeta)}{n}}+\frac{2M_\ell\log(2/\zeta)}{3n}+50\sqrt{2}eG\gamma\lceil\log(n)\rceil\log\left(\frac{2e}{\zeta}\right).
	\end{aligned}
	\end{equation}
	
	Defining $h=h(z_1,\cdots,z_n)=\mathbb{E}_{z}\left[\left(\ell(\mathcal{A}_\bold{w}(S),\mathcal{A}_\bold{v}(S);z)\right)^2\right]$ and $h_i=h_i(z_1,\cdots,z_n)=\sup_{z_i\in\mathcal{Z}}h(z_1,\cdots,z_n)$. We have
	\begin{equation*}
	\begin{aligned}
	\sum_{i=1}^{n}(q-q_i)^2&=\sum_{i=1}^{n}\left(\mathbb{E}_{z}\left[\left(\ell(\mathcal{A}_\bold{w}(S),\mathcal{A}_\bold{v}(S);z)\right)^2\right]-\sup_{z_i\in\mathcal{Z}}\mathbb{E}_{z}\left[\left(\ell(\mathcal{A}_\bold{w}(S),\mathcal{A}_\bold{v}(S);z)\right)^2\right]\right)^2 \\
	&\leq G^2\gamma^2\sum_{i=1}^{n}\left(\mathbb{E}_{z}\left[\ell(\mathcal{A}_\bold{w}(S),\mathcal{A}_\bold{v}(S);z)+\sup_{z_i\in\mathcal{Z}}\ell(\mathcal{A}_\bold{w}(S),\mathcal{A}_\bold{v}(S);z)\right]\right)^2 \\
	&\leq nG^2\gamma^2\left(2\mathbb{E}_{z}\left[\ell(\mathcal{A}_\bold{w}(S),\mathcal{A}_\bold{v}(S);z)\right]+G\gamma\right)^2 \\
	&\leq8G^2n\gamma^2q+2G^4n\gamma^4,
	\end{aligned}
	\end{equation*}
	where the first and the second inequalities hold because of the definition of uniform stability, and the last inequality holds because $(a+b)^2\leq2(a^2+b^2)$ for $a,b>0$.
	
	Via Definition \ref{def3}, it is easy to follow that $h$ is ($8G^2n\gamma^2,2G^4n\gamma^4$)-weakly self-bounded.

	Via Lemma \ref{lem5}, with probability at least $1-\zeta$, we have
	\begin{equation*}
	\begin{aligned}
	&\mathbb{E}_S\left[\mathbb{E}_{z}\left[\left(\ell(\mathcal{A}_\bold{w}(S),\mathcal{A}_\bold{v}(S);z)\right)^2\right]\right]-\mathbb{E}_{z}\left[\left(\ell(\mathcal{A}_\bold{w}(S),\mathcal{A}_\bold{v}(S);z)\right)^2\right] \\
	&\leq\sqrt{\left(16G^2n\gamma^2\mathbb{E}_S\left[\mathbb{E}_{z}\left[\left(\ell(\mathcal{A}_\bold{w}(S),\mathcal{A}_\bold{v}(S);z)\right)^2\right]\right]+4G^4n\gamma^4\right)\log\left(\frac{1}{\zeta}\right)} \\
	&\leq\frac{1}{2}\mathbb{E}_S\left[\mathbb{E}_{z}\left[\left(\ell(\mathcal{A}_\bold{w}(S),\mathcal{A}_\bold{v}(S);z)\right)^2\right]\right]+\frac{G^2\gamma^2}{8}+8G^2n\gamma^2\log\left(\frac{1}{\zeta}\right),
	\end{aligned}
	\end{equation*}
	where the last inequality holds because $\sqrt{ab}\leq\frac{a+b}{2}$ for $a,b>0$.
	
	Noting that $\mathbb{E}_{z}\left[\left(\ell(\mathcal{A}_\bold{w}(S),\mathcal{A}_\bold{v}(S);z)\right)^2\right]\leq M_\ell L(\mathcal{A}_\bold{w}(S),\mathcal{A}_\bold{v}(S))$, the inequality above can be written as
	\begin{equation}\label{eq14}
	\mathbb{E}_S\left[\mathbb{E}_{z}\left[\left(\ell(\mathcal{A}_\bold{w}(S),\mathcal{A}_\bold{v}(S);z)\right)^2\right]\right]-2M_\ell L(\mathcal{A}_\bold{w}(S),\mathcal{A}_\bold{v}(S))\leq\frac{G^2\gamma^2}{4}+16G^2n\gamma^2\log\left(\frac{1}{\zeta}\right).
	\end{equation}
	
	Combining (\ref{eq14}) and (\ref{eq13}) together, with probability at least $1-\zeta$, we have
	\begin{equation}\label{eq31}
	\begin{aligned}
	&L(\mathcal{A}_\bold{w}(S),\mathcal{A}_\bold{v}(S))-L_S(\mathcal{A}_\bold{w}(S),\mathcal{A}_\bold{v}(S)) \\
	&\leq\sqrt{\frac{2\left(2M_\ell L(\mathcal{A}_\bold{w}(S),\mathcal{A}_\bold{v}(S))+\frac{G^2\gamma^2}{4}+16G^2n\gamma^2\log\left(\frac{3}{\zeta}\right)\right)\log\left(\frac{3}{\zeta}\right)}{n}} \\
	&\quad+\frac{2M_\ell\log(3/\zeta)}{3n}+50\sqrt{2}eG\gamma\lceil\log(n)\rceil\log\left(\frac{3e}{\zeta}\right) \\
	&\leq\sqrt{\frac{2\left(\frac{G^2\gamma^2}{4}+16G^2n\gamma^2\log\left(\frac{3}{\zeta}\right)\right)\log\left(\frac{3}{\zeta}\right)}{n}}+\sqrt{\frac{4M_\ell L(\mathcal{A}_\bold{w}(S),\mathcal{A}_\bold{v}(S))\log\left(\frac{3}{\zeta}\right)}{n}} \\
	&\quad+\frac{2M_\ell\log(3/\zeta)}{3n}+50\sqrt{2}eG\gamma\lceil\log(n)\rceil\log\left(\frac{3e}{\zeta}\right) \\
	&\leq\sqrt{\frac{2\left(\frac{G^2\gamma^2}{4}+16G^2n\gamma^2\log\left(\frac{3}{\zeta}\right)\right)\log\left(\frac{3}{\zeta}\right)}{n}}+\iota L(\mathcal{A}_\bold{w}(S),\mathcal{A}_\bold{v}(S))+\frac{4M_\ell\log\left(\frac{3}{\zeta}\right)}{\iota n} \\
	&\quad+\frac{2M_\ell\log(3/\zeta)}{3n}+50\sqrt{2}eG\gamma\lceil\log(n)\rceil\log\left(\frac{3e}{\zeta}\right),
	\end{aligned}
	\end{equation}
	where the second inequality holds because $\sqrt{a+b}\leq\sqrt{a}+\sqrt{b}$ for $a,b>0$ and the last inequality holds because $\sqrt{ab}\leq\iota a+\frac{1}{\iota}b$ for $a,b,\iota>0$.
	
	By rearranging, we have
	\begin{equation*}
	\begin{aligned}
	&L(\mathcal{A}_\bold{w}(S),\mathcal{A}_\bold{v}(S))-\frac{1}{1-\iota}L_S(\mathcal{A}_\bold{w}(S),\mathcal{A}_\bold{v}(S)) \\
	&\leq\sqrt{\frac{\left(G^2\gamma^2+64G^2n\gamma^2\log\left(3/\zeta\right)\right)}{2\left(1-\iota\right)^2n}\log\left(\frac{3}{\zeta}\right)}+\frac{50\sqrt{2}eG\gamma\log(n)}{1-\iota}\log\left(\frac{3e}{\zeta}\right)+\frac{\left(12+2\iota\right)M_\ell}{3\iota\left(1-\iota\right)n}\log\left(\frac{3}{\zeta}\right),
	\end{aligned}
	\end{equation*}
	which ends the proof of part (a).
	
	\textbf{Part (b): The Primal generalization error}.
	
	Denoting $\bold{v}_S^*=\arg\max_{\bold{v}\in\mathcal{V}}L(\mathcal{A}_\bold{w}(S),\bold{v})$ and $\widetilde{\bold{v}}_S^*=\arg\max_{\bold{v}\in\mathcal{V}}L_S(\mathcal{A}_\bold{w}(S),\bold{v})$. We have
	\begin{equation*}
	\begin{aligned}
	nR(\mathcal{A}_\bold{w}(S))-nR_S(\mathcal{A}_\bold{w}(S))&=nL\left(\mathcal{A}_\bold{w}(S),\bold{v}_S^*\right)-nL_S\left(\mathcal{A}_\bold{w}(S),\widetilde{\bold{v}}_S^*\right) \\
	&=\sum_{i=1}^{n}\mathbb{E}_z\left[\ell\left(\mathcal{A}_\bold{w}(S),\bold{v}_S^*;z\right)-\mathbb{E}_{z_i'}\left[\ell\left(\mathcal{A}_\bold{w}(S^{(i)}),\bold{v}_{S^{(i)}}^*;z\right)\right]\right] \\
	&\quad+\sum_{i=1}^{n}\mathbb{E}_{z_i'}\left[\mathbb{E}_z\left[\ell\left(\mathcal{A}_\bold{w}(S^{(i)}),\bold{v}_{S^{(i)}}^*;z\right)\right]-\ell\left(\mathcal{A}_\bold{w}(S^{(i)}),\bold{v}_{S^{(i)}}^*;z_i\right)\right] \\
	&\quad+\sum_{i=1}^{n}\mathbb{E}_{z_i'}\left[\ell\left(\mathcal{A}_\bold{w}(S^{(i)}),\bold{v}_{S^{(i)}}^*;z_i\right)\right]-\sum_{i=1}^{n}\ell\left(\mathcal{A}_\bold{w}(S),\widetilde{\bold{v}}_S^*;z_i\right)
	\end{aligned}
	\end{equation*}
	
	Via Lemma \ref{lem6}, we have
	\begin{equation}\label{eq15}
	\begin{aligned}
	&\ell\left(\mathcal{A}_\bold{w}(S),\bold{v}_S^*;z\right)-\ell\left(\mathcal{A}_\bold{w}(S^{(i)}),\bold{v}_{S^{(i)}}^*;z\right) \\
	&=\ell\left(\mathcal{A}_\bold{w}(S),\bold{v}_S^*;z\right)-\ell\left(\mathcal{A}_\bold{w}(S^{(i)}),\bold{v}_S^*;z\right)+\ell\left(\mathcal{A}_\bold{w}(S^{(i)}),\bold{v}_S^*;z\right)-\ell\left(\mathcal{A}_\bold{w}(S^{(i)}),\bold{v}_{S^{(i)}}^*;z\right) \\
	&\leq G\left\Vert\mathcal{A}_\bold{w}(S)-\mathcal{A}_\bold{w}(S^{(i)})\right\Vert_2+G\left\Vert\bold{v}_S^*-\bold{v}_{S^{(i)}}^*\right\Vert_2 \\
	&\leq\left(1+\frac{L}{\rho}\right)G\left\Vert\mathcal{A}_\bold{w}(S)-\mathcal{A}_\bold{w}(S^{(i)})\right\Vert_2 \\
	&\leq\left(1+\frac{L}{\rho}\right)G\gamma.
	\end{aligned}
	\end{equation}
	
	Recalling $p_i(S)=\mathbb{E}_{z_i'}\left[\mathbb{E}_z\left[\ell(\mathcal{A}_\bold{w}(S^{(i)}),\bold{v}_{S^{(i)}}^*;z)\right]-\ell(\mathcal{A}_\bold{w}(S^{(i)}),\bold{v}_{S^{(i)}}^*;z_i)\right]$, we have
	\begin{equation}\label{eq16}
	\begin{aligned}
	nR(\mathcal{A}_\bold{w}(S))-nR_S(\mathcal{A}_\bold{w}(S))&\leq\left(1+\frac{L}{\rho}\right)Gn\gamma+\sum_{i=1}^{n}p_i(S)+\sum_{i=1}^{n}\mathbb{E}_{z_i'}\left[\ell\left(\mathcal{A}_\bold{w}(S^{(i)}),\bold{v}_{S^{(i)}}^*;z_i\right)\right]-\sum_{i=1}^{n}\ell\left(\mathcal{A}_\bold{w}(S),\widetilde{\bold{v}}_S^*;z_i\right) \\
	&=\left(1+\frac{L}{\rho}\right)Gn\gamma+\sum_{i=1}^{n}p_i(S)+\sum_{i=1}^{n}\mathbb{E}_{z_i'}\left[\ell\left(\mathcal{A}_\bold{w}(S^{(i)}),\bold{v}_{S^{(i)}}^*;z_i\right)-\ell\left(\mathcal{A}_\bold{w}(S),\bold{v}_S^*;z_i\right)\right] \\
	&\quad+\sum_{i=1}^{n}\ell\left(\mathcal{A}_\bold{w}(S),\bold{v}_S^*;z_i\right)-\sum_{i=1}^{n}\ell\left(\mathcal{A}_\bold{w}(S),\widetilde{\bold{v}}_S^*;z_i\right) \\
	&\leq2\left(1+\frac{L}{\rho}\right)Gn\gamma+\sum_{i=1}^{n}p_i(S)+\sum_{i=1}^{n}\ell\left(\mathcal{A}_\bold{w}(S),\bold{v}_S^*;z_i\right)-\sum_{i=1}^{n}\ell\left(\mathcal{A}_\bold{w}(S),\widetilde{\bold{v}}_S^*;z_i\right) \\
	&\leq2\left(1+\frac{L}{\rho}\right)Gn\gamma+\sum_{i=1}^{n}p_i(S),
	\end{aligned}
	\end{equation}
	where the second ineuqality holds because of Lemma \ref{lem6} (similar to (\ref{eq15})) and the last inequality holds because $\widetilde{\bold{v}}_S^*=\arg\max_{\bold{v}\in\mathcal{V}}L_S(\mathcal{A}_\bold{w}(S),\bold{v})$.
	
	Again, defining $q_i(S)=p_i(S)-\mathbb{E}_{S\setminus\{z_i\}}\left[p_i(S)\right]$ and we have $\mathbb{E}_{S\setminus\{z_i\}}\left[q_i(S)\right]=0$ and $\mathbb{E}_{z_i}\left[q_i(S)\right]=0$.
	Moreover, for any $j=1,\cdots,n$ ($j\neq i$) and $z_j'\in\mathcal{Z}$, if we denote $\{z_1,\cdots,z_{j-1},z_j',z_{j+1},\cdots,n\}$ as $S_{j}^{(i)}$
	\begin{equation*}
	\begin{aligned}
	q_i(S)-q_i(z_1,\cdots,z_{j-1},z_j',z_{j+1},\cdots,n)&=p_i(S)-p_i(z_1,\cdots,z_{j-1},z_j',z_{j+1},\cdots,n) \\
	&\quad+\mathbb{E}_{S\setminus\{z_i\}}[p_i(z_1,\cdots,z_{j-1},z_j',z_{j+1},\cdots,n)]-\mathbb{E}_{S\setminus\{z_i\}}[p_i(S)],
	\end{aligned}
	\end{equation*}
	where the first term on the right side
	\begin{equation*}
	\begin{aligned}
	p_i(S)-p_i(z_1,\cdots,z_{j-1},z_j',z_{j+1},\cdots,n)&=\mathbb{E}_{z_i'}\left[\mathbb{E}_z\left[\ell(\mathcal{A}_\bold{w}(S^{(i)}),\bold{v}_{S^{(i)}}^*;z)\right]-\ell(\mathcal{A}_\bold{w}(S^{(i)}),\bold{v}_{S^{(i)}}^*;z_i)\right] \\
	&\quad-\mathbb{E}_{z_i'}\left[\mathbb{E}_z\left[\ell(\mathcal{A}_\bold{w}(S_{j}^{(i)}),\bold{v}_{S_{j}^{(i)}}^*;z)\right]-\ell(\mathcal{A}_\bold{w}(S_{j}^{(i)}),\bold{v}_{S_{j}^{(i)}}^*;z_i)\right] \\
	&\leq2\left(1+\frac{L}{\rho}\right)G\gamma,
	\end{aligned}
	\end{equation*}
	where the inequality holds similar to (\ref{eq15}).
	
	With similar approach, $q_i(S)-q_i(z_1,\cdots,z_{j-1},z_j',z_{j+1},\cdots,n)$ can be upper bounded by $4\left(1+\frac{L}{\rho}\right)G\gamma$, then via Lemma \ref{lem2}, for any $\tau\geq2$, we have
	\begin{equation}\label{eq17}
	\left\Vert\sum_{i=1}^{n}q_i(S)\right\Vert_\tau\leq48\sqrt{2}\left(1+\frac{L}{\rho}\right)G\tau n\gamma\lceil\log(n)\rceil.
	\end{equation}
	
	Plugging the result back into (\ref{eq16}), we have
	\begin{equation}\label{eq18}
	\begin{aligned}
	&\left\Vert nR(\mathcal{A}_\bold{w}(S))-nR_S(\mathcal{A}_\bold{w}(S))-n\mathbb{E}_{S'}\left[R(\mathcal{A}_\bold{w}(S'))\right]+n\mathbb{E}_{S'}\left[R_S(\mathcal{A}_\bold{w}(S'))\right]\right\Vert_\tau \\
	&=\left\Vert nR(\mathcal{A}_\bold{w}(S))-nR_S(\mathcal{A}_\bold{w}(S))-\sum_{i=1}^{n}p_i(S)+\sum_{i=1}^{n}p_i(S)-\sum_{i=1}^n\mathbb{E}_{S\setminus\{z_i\}}\left[p_i(S)\right]\right\Vert_\tau \\
	&\leq\left\Vert nR(\mathcal{A}_\bold{w}(S))-nR_S(\mathcal{A}_\bold{w}(S))-\sum_{i=1}^{n}p_i(S)\right\Vert_\tau+\left\Vert\sum_{i=1}^{n}q_i(S)\right\Vert_\tau \\
	&\leq50\sqrt{2}\left(1+\frac{L}{\rho}\right)G\tau n\gamma\lceil\log(n)\rceil,
	\end{aligned}
	\end{equation}
	where the first equality holds because $n\mathbb{E}_{S'}\left[R(\mathcal{A}_\bold{w}(S'))\right]-n\mathbb{E}_{S'}\left[R_S(\mathcal{A}_\bold{w}(S'))\right]=\sum_{i=1}^n\mathbb{E}_{S\setminus\{z_i\}}\left[p_i(S)\right]$, the first inequality holds because of the definition of $q_i(S)$ and the last inequality holds because of (\ref{eq16}) and (\ref{eq17}).
	
	Via Lemma \ref{lem3}, for $\zeta\in(0,1)$, with probality at least $1-\frac{\zeta}{3}$, we have
	\begin{equation}\label{eq19}
	\begin{aligned}
	R(\mathcal{A}_\bold{w}(S))-R_S(\mathcal{A}_\bold{w}(S))\leq\left|\mathbb{E}_{S'}\left[R(\mathcal{A}_\bold{w}(S'))\right]-\mathbb{E}_{S'}\left[R_S(\mathcal{A}_\bold{w}(S'))\right]\right|+50\sqrt{2}\left(1+\frac{L}{\rho}\right)G\gamma\lceil\log(n)\rceil\log\left(\frac{3e}{\zeta}\right).
	\end{aligned}
	\end{equation}
	
	Like discussed before, we have
	\begin{equation*}
	\mathbb{E}_{z_i}\left[\left(\mathbb{E}_{S'}\left[\ell\left(\mathcal{A}_\bold{w}(S'),\bold{v}_{S'}^*;z_i\right)\right]\right)^2\right]\leq\mathbb{E}_{z_i}\left[\mathbb{E}_{S'}\left[\left(\ell\left(\mathcal{A}_\bold{w}(S'),\bold{v}_{S'}^*;z_i\right)\right)^2\right]\right]=\mathbb{E}_{z}\left[\mathbb{E}_{S}\left[\left(\ell\left(\mathcal{A}_\bold{w}(S),\bold{v}_{S}^*;z\right)\right)^2\right]\right].
	\end{equation*}
	
	And as a result, via Lemma \ref{lem4}, with probality at least $1-\frac{\zeta}{3}$, we have
	\begin{equation}\label{eq20}
	\left|\mathbb{E}_{S'}\left[R(\mathcal{A}_\bold{w}(S'))\right]-\mathbb{E}_{S'}\left[R_S(\mathcal{A}_\bold{w}(S'))\right]\right|\leq\sqrt{\frac{2\mathbb{E}_z\left[\mathbb{E}_S\left[\ell\left(\mathcal{A}_\bold{w}(S),\bold{v}_{S}^*;z\right)^2\right]\right]\log\left(\frac{3}{\zeta}\right)}{n}}+\frac{2M_\ell\log\left(\frac{3}{\zeta}\right)}{3n}.
	\end{equation}
	
	Then, we bound $\mathbb{E}_z\left[\mathbb{E}_S\left[\left(\ell\left(\mathcal{A}_\bold{w}(S),\bold{v}_{S}^*;z\right)\right)^2\right]\right]$.
	
	Defining $g=g(z_1,\cdots,z_n)=\mathbb{E}_z\left[\left(\ell\left(\mathcal{A}_\bold{w}(S),\bold{v}_{S}^*;z\right)\right)^2\right]$, and $g_i=g_i(z_1,\cdots,z_n)=\sup_{z_i\in\mathcal{Z}}g(z_1,\cdots,z_n)$, we have
	\begin{equation*}
	\begin{aligned}
	\sum_{i=1}^{n}\left(g-g_i\right)^2&=\sum_{i=1}^{n}\left(\mathbb{E}_z\left[\left(\ell\left(\mathcal{A}_\bold{w}(S),\bold{v}_{S}^*;z\right)\right)^2\right]-\sup_{z_i\in\mathcal{Z}}\mathbb{E}_z\left[\left(\ell\left(\mathcal{A}_\bold{w}(S),\bold{v}_{S}^*;z\right)\right)^2\right]\right)^2 \\
	&\leq\sum_{i=1}^{n}\left(\mathbb{E}_z\left[\sup_{z_i\in\mathcal{Z}}\left(\ell\left(\mathcal{A}_\bold{w}(S),\bold{v}_{S}^*;z\right)\right)^2-\left(\ell\left(\mathcal{A}_\bold{w}(S),\bold{v}_{S}^*;z\right)\right)^2\right]\right)^2 \\
	&\leq n\left(1+\frac{L}{\rho}\right)^2G^2\gamma^2\left(2\mathbb{E}_z\left[\ell\left(\mathcal{A}_\bold{w}(S),\bold{v}_{S}^*;z\right)\right]+\left(1+\frac{L}{\rho}\right)G\gamma\right)^2 \\
	&\leq8n\left(1+\frac{L}{\rho}\right)^2G^2\gamma^2g+2n\left(1+\frac{L}{\rho}\right)^4G^4\gamma^4,
	\end{aligned}
	\end{equation*}
	where the first inequality holds because of Jensen's inequality, the second ineuqality holds similar to (\ref{eq15}).
	
	Via Definition \ref{def3}, the inequality above implies that $g$ is $(a,b)$-weakly self-bounded where
	\begin{equation*}
	a=8n\left(1+\frac{L}{\rho}\right)^2G^2\gamma^2,\quad b=2n\left(1+\frac{L}{\rho}\right)^4G^4\gamma^4.
	\end{equation*}
	
	As a result, via Lemma \ref{lem6}, with probability at least $1-\frac{\zeta}{3}$, we have
	\begin{equation*}
	\begin{aligned}
	&\mathbb{E}_S\left[\mathbb{E}_z\left[\left(\ell\left(\mathcal{A}_\bold{w}(S),\bold{v}_{S}^*;z\right)\right)^2\right]\right]-\mathbb{E}_z\left[\left(\ell\left(\mathcal{A}_\bold{w}(S),\bold{v}_{S}^*;z\right)\right)^2\right] \\
	&\leq\sqrt{\left(16n\left(1+\frac{L}{\rho}\right)^2G^2\gamma^2\mathbb{E}_S\left[\mathbb{E}_z\left[\left(\ell\left(\mathcal{A}_\bold{w}(S),\bold{v}_{S}^*;z\right)\right)^2\right]\right]+4n\left(1+\frac{L}{\rho}\right)^4G^4\gamma^4\right)\log\left(\frac{3}{\zeta}\right)} \\
	&\leq\frac{1}{2}\mathbb{E}_S\left[\mathbb{E}_z\left[\left(\ell\left(\mathcal{A}_\bold{w}(S),\bold{v}_{S}^*;z\right)\right)^2\right]\right]+\frac{1}{8}\left(1+\frac{L}{\rho}\right)^2G^2\gamma^2+8n\left(1+\frac{L}{\rho}\right)^2G^2\gamma^2\log\left(\frac{3}{\zeta}\right).
	\end{aligned}
	\end{equation*}
	
	Noting that $\mathbb{E}_z\left[\left(\ell\left(\mathcal{A}_\bold{w}(S),\bold{v}_{S}^*;z\right)\right)^2\right]\leq M_\ell L\left(\mathcal{A}_\bold{w}(S),\bold{v}_{S}^*\right)=M_\ell R\left(\mathcal{A}_\bold{w}(S)\right)$, we have
	\begin{equation}\label{eq21}
	\mathbb{E}_S\left[\mathbb{E}_z\left[\left(\ell\left(\mathcal{A}_\bold{w}(S),\bold{v}_{S}^*;z\right)\right)^2\right]\right]-2M_\ell R\left(\mathcal{A}_\bold{w}(S)\right)\leq\frac{1}{4}\left(1+\frac{L}{\rho}\right)^2G^2\gamma^2+16n\left(1+\frac{L}{\rho}\right)^2G^2\gamma^2\log\left(\frac{3}{\zeta}\right).
	\end{equation}
	
	Combining (\ref{eq19}), (\ref{eq20}), and (\ref{eq21}) together, then for all $\iota>0$, with probability at least $1-\zeta$, we have
	\begin{equation}\label{eq22}
	\begin{aligned}
	&R\left(\mathcal{A}_\bold{w}(S)\right)-R_S\left(\mathcal{A}_\bold{w}(S)\right) \\
	&\leq50\sqrt{2}\left(1+\frac{L}{\rho}\right)G\gamma\lceil\log(n)\rceil\log\left(\frac{3e}{\zeta}\right)+\frac{2M_\ell\log\left(\frac{3}{\zeta}\right)}{3n} \\
	&\quad+\sqrt{\frac{2\left(\frac{1}{4}\left(1+\frac{L}{\rho}\right)^2G^2\gamma^2+16n\left(1+\frac{L}{\rho}\right)^2G^2\gamma^2\log\left(\frac{3}{\zeta}\right)+2M_\ell R\left(\mathcal{A}_\bold{w}(S)\right)\right)\log\left(\frac{3}{\zeta}\right)}{n}} \\
	&\leq50\sqrt{2}\left(1+\frac{L}{\rho}\right)G\gamma\lceil\log(n)\rceil\log\left(\frac{3e}{\zeta}\right)+\frac{2M_\ell\log\left(\frac{3}{\zeta}\right)}{3n} \\
	&\quad+\sqrt{\frac{\left(\frac{1}{2}\left(1+\frac{L}{\rho}\right)^2G^2\gamma^2+32n\left(1+\frac{L}{\rho}\right)^2G^2\gamma^2\log\left(\frac{3}{\zeta}\right)\right)\log\left(\frac{3}{\zeta}\right)}{n}}+\iota R\left(\mathcal{A}_\bold{w}(S)\right)+\frac{4M_\ell\log\left(\frac{3}{\zeta}\right)}{\iota n},
	\end{aligned}
	\end{equation}
	where the last inequality holds because $\sqrt{a+b}\leq\sqrt{a}+\sqrt{b}$ and $\sqrt{ab}\leq\iota a+\frac{1}{\iota}b$ for all $a,b,\iota>0$.
	
	By rearranging, with probability at least $1-\zeta$, we have
	\begin{equation*}
	\begin{aligned}
	&R\left(\mathcal{A}_\bold{w}(S)\right)-\frac{1}{1-\iota}R_S\left(\mathcal{A}_\bold{w}(S)\right) \\
	&\leq\sqrt{\frac{\left(1+L/\rho\right)^2G^2\gamma^2\left(1+64n\log\left(3/\zeta\right)\right)}{2(1-\iota)^2n}\log\left(\frac{3}{\zeta}\right)}+\frac{50\sqrt{2}\left(1+L/\rho\right)G\gamma\log(n)}{1-\iota}\log\left(\frac{3e}{\zeta}\right)+\frac{(12+2\iota)M_\ell}{3\iota(1-\iota)n}\log\left(\frac{3}{\zeta}\right),
	\end{aligned}
	\end{equation*}
	which completes the proof of Part (b).

	\textbf{Part (c): The primal excess population risk}.
	
	In this section, we denote $\bold{w}^*=\arg\min_{\bold{w}\in\mathcal{W}}R(\bold{w})$. Then, we have
	\begin{equation*}
	\begin{aligned}
	R\left(\mathcal{A}_\bold{w}(S)\right)-R\left(\bold{w}^*\right)&=\underbrace{R\left(\mathcal{A}_\bold{w}(S)\right)-R_S\left(\mathcal{A}_\bold{w}(S)\right)}_{A}+\underbrace{R_S\left(\mathcal{A}_\bold{w}(S)\right)-L_S\left(\bold{w}^*,\mathcal{A}_\bold{v}(S)\right)}_{B} \\
	&\quad+\underbrace{L_S\left(\bold{w}^*,\mathcal{A}_\bold{v}(S)\right)-L\left(\bold{w}^*,\mathcal{A}_\bold{v}(S)\right)}_{C}+\underbrace{L\left(\bold{w}^*,\mathcal{A}_\bold{v}(S)\right)-R\left(\bold{w}^*\right)}_{D}.
	\end{aligned}
	\end{equation*}
	
	For part $A$, according to (\ref{eq22}), with probability at least $1-\zeta$, we have
	\begin{equation*}
	\begin{aligned}
	R\left(\mathcal{A}_\bold{w}(S)\right)-R_S\left(\mathcal{A}_\bold{w}(S)\right)&\leq50\sqrt{2}\left(1+\frac{L}{\rho}\right)G\gamma\lceil\log(n)\rceil\log\left(\frac{3e}{\zeta}\right)+\frac{2M_\ell\log\left(\frac{3}{\zeta}\right)}{3n} \\
	&\quad+\sqrt{\frac{\left(\frac{1}{2}\left(1+\frac{L}{\rho}\right)^2G^2\gamma^2+32n\left(1+\frac{L}{\rho}\right)^2G^2\gamma^2\log\left(\frac{3}{\zeta}\right)+4M_\ell R\left(\mathcal{A}_\bold{w}(S)\right)\right)\log\left(\frac{3}{\zeta}\right)}{n}}.
	\end{aligned}
	\end{equation*}
	
	For part $B$, we have
	\begin{equation*}
	R_S\left(\mathcal{A}_\bold{w}(S)\right)-L_S\left(\bold{w}^*,\mathcal{A}_\bold{v}(S)\right)\leq R_S\left(\mathcal{A}_\bold{w}(S)\right)-L_S\left(\widetilde{\bold{w}}_S^*,\mathcal{A}_\bold{v}(S)\right)=\triangle_S^s\left(\mathcal{A}_\bold{w}(S),\mathcal{A}_\bold{v}(S)\right),
	\end{equation*}
	where $\widetilde{\bold{w}}_S^*=\arg\min_{\bold{w}\in\mathcal{W}}L_S\left(w,\mathcal{A}_\bold{v}(S)\right)$.
	
	For part $C$, according to (\ref{eq31}), with probability at least $1-\zeta$, we have
	\begin{equation*}
	\begin{aligned}
	L_S(\bold{w}^*,\mathcal{A}_\bold{v}(S))-L(\bold{w}^*,\mathcal{A}_\bold{v}(S))&\leq\frac{2M_\ell}{3n}\log\left(\frac{3}{\zeta}\right)+50\sqrt{2}eG\gamma\lceil\log(n)\rceil\log\left(\frac{3e}{\zeta}\right) \\
	&\quad+\sqrt{\frac{\left(4M_\ell L(\bold{w}^*,\mathcal{A}_\bold{v}(S))+\frac{G^2\gamma^2}{2}+32G^2n\gamma^2\log\left(\frac{3}{\zeta}\right)\right)\log\left(\frac{3}{\zeta}\right)}{n}}.
	\end{aligned}
	\end{equation*}
	
	For part $D$, it is easy to follow that $L\left(\bold{w}^*,\mathcal{A}_\bold{v}(S)\right)-R\left(\bold{w}^*\right)\leq0$.
	
	Combining parts $A,B,C,D$ together, with probability at least $1-\zeta$, we have
	\begin{equation*}
	\begin{aligned}
	R\left(\mathcal{A}_\bold{w}(S)\right)-R\left(\bold{w}^*\right)&\leq50\sqrt{2}\left(1+e+\frac{L}{\rho}\right)G\gamma\lceil\log(n)\rceil\log\left(\frac{6e}{\zeta}\right)+\frac{4M_\ell}{3n}\log\left(\frac{6}{\zeta}\right)+\triangle_S^s\left(\mathcal{A}_\bold{w}(S),\mathcal{A}_\bold{v}(S)\right) \\
	&\quad+\sqrt{\frac{\left(\frac{1}{2}\left(1+\frac{L}{\rho}\right)^2G^2\gamma^2+32n\left(1+\frac{L}{\rho}\right)^2G^2\gamma^2\log\left(\frac{6}{\zeta}\right)+4M_\ell R\left(\mathcal{A}_\bold{w}(S)\right)\right)\log\left(\frac{6}{\zeta}\right)}{n}} \\
	&\quad+\sqrt{\frac{\left(4M_\ell R(\bold{w}^*)+\frac{G^2\gamma^2}{2}+32G^2n\gamma^2\log\left(\frac{6}{\zeta}\right)\right)\log\left(\frac{6}{\zeta}\right)}{n}} \\
	&\leq50\sqrt{2}\left(1+e+\frac{L}{\rho}\right)G\gamma\lceil\log(n)\rceil\log\left(\frac{6e}{\zeta}\right)+\frac{4M_\ell}{3n}\log\left(\frac{6}{\zeta}\right)+\triangle_S^s\left(\mathcal{A}_\bold{w}(S),\mathcal{A}_\bold{v}(S)\right) \\
	&\quad+\sqrt{\frac{\left(\frac{1}{2}\left(1+\frac{L}{\rho}\right)^2G^2\gamma^2+32n\left(1+\frac{L}{\rho}\right)^2G^2\gamma^2\log\left(\frac{6}{\zeta}\right)\right)\log\left(\frac{6}{\zeta}\right)}{n}}+\iota R\left(\mathcal{A}_\bold{w}(S)\right) \\
	&\quad+\sqrt{\frac{\left(\frac{G^2\gamma^2}{2}+32G^2n\gamma^2\log\left(\frac{6}{\zeta}\right)\right)\log\left(\frac{6}{\zeta}\right)}{n}}+\iota R(\bold{w}^*)+\frac{8M_\ell}{\iota n}\log\left(\frac{6}{\zeta}\right),
	\end{aligned}
	\end{equation*}
	where the first inequality holds because $L(\bold{w}^*,\mathcal{A}_\bold{v}(S))\leq R(\bold{w}^*)$ and the last inequality holds because $\sqrt{a+b}\leq\sqrt{a}+\sqrt{b}$ and $\sqrt{ab}\leq\iota a+\frac{1}{\iota}b$ for all $a,b,\iota>0$.
	
	By rearranging, for all $\iota>0$, with probability at least $1-\zeta$, we have
	\begin{equation*}
	\begin{aligned}
	&R\left(\mathcal{A}_\bold{w}(S)\right)-\frac{1+\iota}{1-\iota}\inf_{\bold{w}\in\mathcal{W}}R\left(\bold{w}\right) \\
	&\leq\sqrt{\frac{\left(1+L/\rho\right)^2G^2\gamma^2\left(1+64n\log\left(6/\zeta\right)\right)}{2(1-\iota)^2n}\log\left(\frac{6}{\zeta}\right)}+\sqrt{\frac{\left(G^2\gamma^2+64G^2n\gamma^2\log\left(6/\zeta\right)\right)}{2(1-\iota)^2n}\log\left(\frac{6}{\zeta}\right)} \\
	&\quad+\frac{50\sqrt{2}\left(1+e+L/\rho\right)G\gamma\log(n)}{1-\iota}\log\left(\frac{6e}{\zeta}\right)+\frac{(24+4\iota)M_\ell}{3\iota(1-\iota)n}\log\left(\frac{6}{\zeta}\right)+\frac{1}{1-\iota}\triangle_S^s\left(\mathcal{A}_\bold{w}(S),\mathcal{A}_\bold{v}(S)\right).
	\end{aligned}
	\end{equation*}
	
	The proof of part (c) completes.
	
	\textbf{Part (d): The strong PD population risk}.
	
	Denoting $\bold{w}_S^*=\arg\min_{\bold{w}\in\mathcal{W}}L(\bold{w},\mathcal{A}_\bold{v}(S))$, $\widetilde{\bold{w}}_S^*=\arg\min_{\bold{w}\in\mathcal{W}}L_S(\bold{w},\mathcal{A}_\bold{v}(S))$, $\bold{v}_S^*=\arg\max_{\bold{v}\in\mathcal{V}}L(\mathcal{A}_\bold{w}(S),\bold{v})$, and $\widetilde{\bold{v}}_S^*=\arg\max_{\bold{v}\in\mathcal{V}}L_S(\mathcal{A}_\bold{w}(S),\bold{v})$. We have
	\begin{equation*}
	\begin{aligned}
	\triangle^s\left(\mathcal{A}_\bold{w}(S),\mathcal{A}_\bold{v}(S)\right)&=\sup_{\bold{v}\in\mathcal{V}}L\left(\mathcal{A}_\bold{w}(S),\bold{v}\right)-\inf_{\bold{w}\in\mathcal{W}}L\left(\bold{w},\mathcal{A}_\bold{v}(S)\right) \\
	&=\underbrace{L\left(\mathcal{A}_\bold{w}(S),\bold{v}_S^*\right)-L_S\left(\mathcal{A}_\bold{w}(S),\widetilde{\bold{v}}_S^*\right)+\mathbb{E}_{S'}\left[L_S\left(\mathcal{A}_\bold{w}(S'),\bold{v}_{S'}^*\right)\right]-\mathbb{E}_S\left[L\left(\mathcal{A}_\bold{w}(S),\bold{v}_S^*\right)\right]}_{A} \\
	&\quad\underbrace{+\mathbb{E}_S\left[L\left(\bold{w}_S^*,\mathcal{A}_\bold{v}(S)\right)\right]-\mathbb{E}_{S'}\left[L_S\left(\bold{w}_{S'}^*,\mathcal{A}_\bold{v}(S')\right)\right]+L_S\left(\widetilde{\bold{w}}_{S}^*,\mathcal{A}_\bold{v}(S)\right)-L\left(\bold{w}_{S}^*,\mathcal{A}_\bold{v}(S)\right)}_{B} \\
	&\quad\underbrace{-\mathbb{E}_{S'}\left[L_S\left(\mathcal{A}_\bold{w}(S'),\bold{v}_{S'}^*\right)\right]+\mathbb{E}_S\left[L\left(\mathcal{A}_\bold{w}(S),\bold{v}_S^*\right)\right]-\mathbb{E}_S\left[L\left(\bold{w}_S^*,\mathcal{A}_\bold{v}(S)\right)\right]+\mathbb{E}_{S'}\left[L_S\left(\bold{w}_{S'}^*,\mathcal{A}_\bold{v}(S')\right)\right]}_{C} \\
	&\quad\underbrace{+L_S\left(\mathcal{A}_\bold{w}(S),\widetilde{\bold{v}}_S^*\right)-L_S\left(\widetilde{\bold{w}}_{S}^*,\mathcal{A}_\bold{v}(S)\right)}_{D}.
	\end{aligned}
	\end{equation*}
	
	For part $A$, via inequality (\ref{eq18}), for $\tau\geq2$, we have
	\begin{equation}\label{eq23}
	\begin{aligned}
	&\left\Vert R(\mathcal{A}_\bold{w}(S))-R_S(\mathcal{A}_\bold{w}(S))-\mathbb{E}_{S'}\left[R(\mathcal{A}_\bold{w}(S'))\right]+\mathbb{E}_{S'}\left[R_S(\mathcal{A}_\bold{w}(S'))\right]\right\Vert_\tau \\
	&=\left\Vert L\left(\mathcal{A}_\bold{w}(S),\bold{v}_S^*\right)-L_S\left(\mathcal{A}_\bold{w}(S),\widetilde{\bold{v}}_S^*\right)+\mathbb{E}_{S'}\left[L_S\left(\mathcal{A}_\bold{w}(S'),\bold{v}_{S'}^*\right)\right]-\mathbb{E}_S\left[L\left(\mathcal{A}_\bold{w}(S),\bold{v}_S^*\right)\right]\right\Vert_\tau \\
	&\leq50\sqrt{2}\left(1+\frac{L}{\rho}\right)G\tau\gamma\lceil\log(n)\rceil.
	\end{aligned}
	\end{equation}
	
	For part $B$, we first analyze
	\begin{equation}\label{eq24}
	\begin{aligned}
	&nL_S\left(\widetilde{\bold{w}}_{S}^*,\mathcal{A}_\bold{v}(S)\right)-nL\left(\bold{w}_{S}^*,\mathcal{A}_\bold{v}(S)\right) \\
	&=nL_S\left(\widetilde{\bold{w}}_{S}^*,\mathcal{A}_\bold{v}(S)\right)-\sum_{i=1}^{n}\mathbb{E}_z\left[\ell\left(\bold{w}_{S}^*,\mathcal{A}_\bold{v}(S);z\right)-\mathbb{E}_{z_i'}\left[\ell\left(\bold{w}_{S^{(i)}}^*,\mathcal{A}_\bold{v}(S^{(i)});z\right)\right]\right] \\
	&\quad+\sum_{i=1}^{n}\mathbb{E}_{z_i'}\left[\ell\left(\bold{w}_{S^{(i)}}^*,\mathcal{A}_\bold{v}(S^{(i)});z_i\right)-\mathbb{E}_z\left[\ell\left(\bold{w}_{S^{(i)}}^*,\mathcal{A}_\bold{v}(S^{(i)});z\right)\right]\right]-\sum_{i=1}^{n}\mathbb{E}_{z_i'}\left[\ell\left(\bold{w}_{S^{(i)}}^*,\mathcal{A}_\bold{v}(S^{(i)});z_i\right)\right],
	\end{aligned}
	\end{equation}
	in which
	\begin{equation}\label{eq25}
	\begin{aligned}
	\sum_{i=1}^{n}\mathbb{E}_{z_i'}\left[\ell\left(\bold{w}_{S^{(i)}}^*,\mathcal{A}_\bold{v}(S^{(i)});z_i\right)\right]&=\sum_{i=1}^{n}\mathbb{E}_{z_i'}\left[\ell\left(\bold{w}_{S^{(i)}}^*,\mathcal{A}_\bold{v}(S^{(i)});z_i\right)-\ell\left(\bold{w}_{S}^*,\mathcal{A}_\bold{v}(S);z_i\right)+\ell\left(\bold{w}_{S}^*,\mathcal{A}_\bold{v}(S);z_i\right)\right] \\
	&\geq nL_S\left(\widetilde{\bold{w}}_{S}^*,\mathcal{A}_\bold{v}(S)\right)+\sum_{i=1}^{n}\mathbb{E}_{z_i'}\left[\ell\left(\bold{w}_{S^{(i)}}^*,\mathcal{A}_\bold{v}(S^{(i)});z_i\right)-\ell\left(\bold{w}_{S}^*,\mathcal{A}_\bold{v}(S);z_i\right)\right],
	\end{aligned}
	\end{equation}
	where the last inequality holds because $\widetilde{\bold{w}}_S^*=\arg\min_{\bold{w}\in\mathcal{W}}L_S(\bold{w},\mathcal{A}_\bold{v}(S))$.
	
	Defining $p_i(S)=\mathbb{E}_{z_i'}\left[\ell\left(\bold{w}_{S^{(i)}}^*,\mathcal{A}_\bold{v}(S^{(i)});z_i\right)-\mathbb{E}_z\left[\ell\left(\bold{w}_{S^{(i)}}^*,\mathcal{A}_\bold{v}(S^{(i)});z\right)\right]\right]$ and plugging (\ref{eq25}) back into (\ref{eq24}), we have
	\begin{equation*}
	\begin{aligned}
	nL_S\left(\widetilde{\bold{w}}_{S}^*,\mathcal{A}_\bold{v}(S)\right)-nL\left(\bold{w}_{S}^*,\mathcal{A}_\bold{v}(S)\right)&\leq-\sum_{i=1}^{n}\mathbb{E}_z\left[\ell\left(\bold{w}_{S}^*,\mathcal{A}_\bold{v}(S);z\right)-\mathbb{E}_{z_i'}\left[\ell\left(\bold{w}_{S^{(i)}}^*,\mathcal{A}_\bold{v}(S^{(i)});z\right)\right]\right]+\sum_{i=1}^{n}p_i(S) \\
	&\quad-\sum_{i=1}^{n}\mathbb{E}_{z_i'}\left[\ell\left(\bold{w}_{S^{(i)}}^*,\mathcal{A}_\bold{v}(S^{(i)});z_i\right)-\ell\left(\bold{w}_{S}^*,\mathcal{A}_\bold{v}(S);z_i\right)\right].
	\end{aligned}
	\end{equation*}
	
	Similar to (\ref{eq15}), via Lemma \ref{lem6}, we have
	\begin{equation*}
	\begin{aligned}
	nL_S\left(\widetilde{\bold{w}}_{S}^*,\mathcal{A}_\bold{v}(S)\right)-nL\left(\bold{w}_{S}^*,\mathcal{A}_\bold{v}(S)\right)&\leq\sum_{i=1}^{n}p_i(S)+2n\left(1+\frac{L}{\rho}\right)G\gamma.
	\end{aligned}
	\end{equation*}
	
	Again, we define $q_i(S)=p_i(S)-\mathbb{E}_{S\setminus\{z_i\}}\left[p_i(S)\right]$ and have $\mathbb{E}_{S\setminus\{z_i\}}\left[q_i(S)\right]=\mathbb{E}_{z_i}\left[q_i(S)\right]=0$.
	Similarly, for $j=1,\cdots,n$ ($j\neq i$) and $z_j'\in\mathcal{Z}$, we have
	\begin{equation*}
	q_i(S)-q_i(z_1,\cdots,z_{j-1},z_j',z_{j+1},\cdots,z_n)\leq4\left(1+\frac{L}{\rho}\right)G\gamma,
	\end{equation*}
	which implies
	\begin{equation*}
	\left\Vert\sum_{i=1}^{n}q_i(S)\right\Vert_\tau\leq48\sqrt{2}\left(1+\frac{L}{\rho}\right)G\tau n\gamma\lceil\log(n)\rceil,
	\end{equation*}
	for $\tau\geq2$, via Lemma \ref{lem1}.
	
	Then for part $B$, like in (\ref{eq18}), for $\tau\geq2$, we have
	\begin{equation}\label{eq26}
	\begin{aligned}
	&\left\Vert\mathbb{E}_S\left[L\left(\bold{w}_S^*,\mathcal{A}_\bold{v}(S)\right)\right]-\mathbb{E}_{S'}\left[L_S\left(\bold{w}_{S'}^*,\mathcal{A}_\bold{v}(S')\right)\right]+L_S\left(\widetilde{\bold{w}}_{S}^*,\mathcal{A}_\bold{v}(S)\right)-L\left(\bold{w}_{S}^*,\mathcal{A}_\bold{v}(S)\right)\right\Vert_\tau \\
	&\leq\left\Vert L_S\left(\widetilde{\bold{w}}_{S}^*,\mathcal{A}_\bold{v}(S)\right)-L\left(\bold{w}_{S}^*,\mathcal{A}_\bold{v}(S)\right)-\frac{1}{n}\sum_{i=1}^{n}p_i(S)\right\Vert_\tau+\left\Vert\frac{1}{n}\sum_{i=1}^{n}q_i(S)\right\Vert_\tau \\
	&\leq50\sqrt{2}\left(1+\frac{L}{\rho}\right)G\tau\gamma\lceil\log(n)\rceil.
	\end{aligned}
	\end{equation}
	
	For part $C$, we have
	\begin{equation}\label{eq27}
	\begin{aligned}
	&-\mathbb{E}_{S'}\left[L_S\left(\mathcal{A}_\bold{w}(S'),\bold{v}_{S'}^*\right)\right]+\mathbb{E}_S\left[L\left(\mathcal{A}_\bold{w}(S),\bold{v}_S^*\right)\right]-\mathbb{E}_S\left[L\left(\bold{w}_S^*,\mathcal{A}_\bold{v}(S)\right)\right]+\mathbb{E}_{S'}\left[L_S\left(\bold{w}_{S'}^*,\mathcal{A}_\bold{v}(S')\right)\right] \\
	&=\frac{1}{n}\sum_{i=1}^{n}\left(-\mathbb{E}_{S'}\left[\ell\left(\mathcal{A}_\bold{w}(S'),\bold{v}_{S'}^*;z_i\right)\right]+\mathbb{E}_{S'}\left[\ell\left(\bold{w}_{S'}^*,\mathcal{A}_\bold{v}(S');z_i\right)\right]\right)+\mathbb{E}_S\left[L\left(\mathcal{A}_\bold{w}(S),\bold{v}_S^*\right)\right]-\mathbb{E}_S\left[L\left(\bold{w}_S^*,\mathcal{A}_\bold{v}(S)\right)\right].
	\end{aligned}
	\end{equation}

	Defining $g_i=u+v_i$, where $u=\mathbb{E}_S\left[L\left(\mathcal{A}_\bold{w}(S),\bold{v}_S^*\right)\right]-\mathbb{E}_S\left[L\left(\bold{w}_S^*,\mathcal{A}_\bold{v}(S)\right)\right]$, and $v_i=\mathbb{E}_{S'}\left[\ell\left(\bold{w}_{S'}^*,\mathcal{A}_\bold{v}(S');z_i\right)\right]-\mathbb{E}_{S'}\left[\ell\left(\mathcal{A}_\bold{w}(S'),\bold{v}_{S'}^*;z_i\right)\right]$.
	Then $g_i$ is the element in (\ref{eq27}).
	Besides, we have $\mathbb{E}\left[g_i^2\right]\leq\mathbb{E}\left[\left(g_i-\mathbb{E}\left[g_i\right]\right)^2\right]$.
	Noting that $\mathbb{E}_S\left[\mathbb{E}_{S'}\left[L_S\left(\bold{w}_{S'}^*,\mathcal{A}_\bold{v}(S')\right)\right]-\mathbb{E}_{S'}\left[L_S\left(\mathcal{A}_\bold{w}(S'),\bold{v}_{S'}^*\right)\right]\right]=\mathbb{E}_S\left[L\left(\bold{w}_S^*,\mathcal{A}_\bold{v}(S)\right)\right]-\mathbb{E}_S\left[L\left(\mathcal{A}_\bold{w}(S),\bold{v}_S^*\right)\right]=-u$, so $\mathbb{E}\left[g_i^2\right]\leq\mathbb{E}\left[\left(g_i-\mathbb{E}\left[g_i\right]\right)^2\right]=\mathbb{E}\left[v_i^2\right]$, i,e., $\mathbb{E}\left[g_i^2\right]$ can be bounded by $\mathbb{E}\left[v_i^2\right]$.
	
	For $\mathbb{E}\left[v_i^2\right]$, we have
	\begin{equation}\label{eq28}
	\begin{aligned}
	\mathbb{E}\left[v_i^2\right]&=\mathbb{E}\left[\left(\mathbb{E}_{S'}\left[\ell\left(\bold{w}_{S'}^*,\mathcal{A}_\bold{v}(S');z_i\right)\right]-\mathbb{E}_{S'}\left[\ell\left(\mathcal{A}_\bold{w}(S'),\bold{v}_{S'}^*;z_i\right)\right]\right)^2\right] \\
	&\leq\mathbb{E}\left[\mathbb{E}_{S'}\left[G^2\left(\left\Vert\bold{w}_{S'}^*-\mathcal{A}_\bold{w}(S')\right\Vert_2+\left\Vert\mathcal{A}_\bold{v}(S')-\bold{v}_{S'}^*\right\Vert_2\right)^2\right]\right] \\
	&\leq2G^2\mathbb{E}\left[\mathbb{E}_{S'}\left[\left\Vert\bold{w}_{S'}^*-\mathcal{A}_\bold{w}(S')\right\Vert_2^2+\left\Vert\mathcal{A}_\bold{v}(S')-\bold{v}_{S'}^*\right\Vert_2^2\right]\right],
	\end{aligned}
	\end{equation}
	where the first inequality holds because of Jensen's inequality and the $G$-Lipschitz property of $\ell(\cdot,\cdot;\cdot)$, and the second inequality holds because $(a+b)^2\leq2(a^2+b^2)$.
	
	Noting that $L(\cdot,\cdot)$ is $\rho$-strongly-convex-strongly-concave, we have
	\begin{equation*}
	\begin{aligned}
	\frac{\rho}{2}\left\Vert\mathcal{A}_\bold{v}(S')-\bold{v}_{S'}^*\right\Vert_2^2&\leq L\left(\bold{w}_{S'}^*,\bold{v}_{S'}^*\right)-L\left(\bold{w}_{S'}^*,\mathcal{A}_\bold{v}(S')\right), \\
	\frac{\rho}{2}\left\Vert\mathcal{A}_\bold{w}(S')-\bold{w}_{S'}^*\right\Vert_2^2&\leq L\left(\mathcal{A}_\bold{w}(S'),\bold{v}_{S'}^*\right)-L\left(\bold{w}_{S'}^*,\bold{v}_{S'}^*\right).
	\end{aligned}
	\end{equation*}
	
	As a result
	\begin{equation*}
	\left\Vert\mathcal{A}_\bold{v}(S')-\bold{v}_{S'}^*\right\Vert_2^2+\left\Vert\mathcal{A}_\bold{w}(S')-\bold{w}_{S'}^*\right\Vert_2^2\leq\frac{2}{\rho}\left(L\left(\mathcal{A}_\bold{w}(S'),\bold{v}_{S'}^*\right)-L\left(\bold{w}_{S'}^*,\mathcal{A}_\bold{v}(S')\right)\right).
	\end{equation*}
	
	Plugging this result back into (\ref{eq28}), we have
	\begin{equation*}
	\begin{aligned}
	\mathbb{E}\left[\left(\mathbb{E}_{S'}\left[\ell\left(\bold{w}_{S'}^*,\mathcal{A}_\bold{v}(S');z_i\right)\right]-\mathbb{E}_{S'}\left[\ell\left(\mathcal{A}_\bold{w}(S'),\bold{v}_{S'}^*;z_i\right)\right]\right)^2\right]&\leq\frac{4G^2}{\rho}\mathbb{E}_{S'}\left[L\left(\mathcal{A}_\bold{w}(S'),\bold{v}_{S'}^*\right)-L\left(\bold{w}_{S'}^*,\mathcal{A}_\bold{v}(S')\right)\right] \\
	&=\frac{4G^2}{\rho}\mathbb{E}_{S}\left[L\left(\mathcal{A}_\bold{w}(S),\bold{v}_{S}^*\right)-L\left(\bold{w}_{S}^*,\mathcal{A}_\bold{v}(S)\right)\right].
	\end{aligned}
	\end{equation*}
	
	Then, via Lemma \ref{lem7}, for part $C$, for $\tau\geq2$, we have
	\begin{equation}\label{eq29}
	\begin{aligned}
	&\left\Vert\frac{1}{n}\sum_{i=1}^{n}\left(-\mathbb{E}_{S'}\left[\ell\left(\mathcal{A}_\bold{w}(S'),\bold{v}_{S'}^*;z_i\right)\right]+\mathbb{E}_{S'}\left[\ell\left(\bold{w}_{S'}^*,\mathcal{A}_\bold{v}(S');z_i\right)\right]\right)+\mathbb{E}_S\left[L\left(\mathcal{A}_\bold{w}(S),\bold{v}_S^*\right)\right]-\mathbb{E}_S\left[L\left(\bold{w}_S^*,\mathcal{A}_\bold{v}(S)\right)\right]\right\Vert_\tau \\
	&\leq6\sqrt{\frac{4G^2\tau}{\rho n}\mathbb{E}_{S}\left[L\left(\mathcal{A}_\bold{w}(S),\bold{v}_{S}^*\right)-L\left(\bold{w}_{S}^*,\mathcal{A}_\bold{v}(S)\right)\right]}+\frac{8M_\ell\tau}{n}.
	\end{aligned}
	\end{equation}
	
	For part $D$, $L_S\left(\mathcal{A}_\bold{w}(S),\widetilde{\bold{v}}_S^*\right)-L_S\left(\widetilde{\bold{w}}_{S}^*,\mathcal{A}_\bold{v}(S)\right)$ is exactly the strong PD empirical risk $\triangle_S^s\left(\mathcal{A}_\bold{w}(S),\mathcal{A}_\bold{v}(S)\right)$.
	
	Plugging (\ref{eq23}), (\ref{eq26}), and (\ref{eq29}) back into $\triangle^s\left(\mathcal{A}_\bold{w}(S),\mathcal{A}_\bold{v}(S)\right)$, then for all $\tau\geq2$, we have
	\begin{equation}\label{eq30}
	\begin{aligned}
	&\left\Vert\triangle^s\left(\mathcal{A}_\bold{w}(S),\mathcal{A}_\bold{v}(S)\right)-\triangle_S^s\left(\mathcal{A}_\bold{w}(S),\mathcal{A}_\bold{v}(S)\right)\right\Vert_\tau \\
	&\leq100\sqrt{2}\left(1+\frac{L}{\rho}\right)G\tau\gamma\lceil\log(n)\rceil+12\sqrt{\frac{G^2\tau}{\rho n}\mathbb{E}_{S}\left[L\left(\mathcal{A}_\bold{w}(S),\bold{v}_{S}^*\right)-L\left(\bold{w}_{S}^*,\mathcal{A}_\bold{v}(S)\right)\right]}+\frac{8M_\ell\tau}{n} \\
	&\leq100\sqrt{2}\left(1+\frac{L}{\rho}\right)G\tau\gamma\lceil\log(n)\rceil+\iota\mathbb{E}_{S}\left[L\left(\mathcal{A}_\bold{w}(S),\bold{v}_{S}^*\right)-L\left(\bold{w}_{S}^*,\mathcal{A}_\bold{v}(S)\right)\right]+\frac{144G^2\tau}{\rho\iota n}+\frac{8M_\ell\tau}{n},
	\end{aligned}
	\end{equation}
	where the last inequality holds because $\sqrt{ab}\leq\iota a+\frac{1}{\iota}b$ for $a,b,\iota>0$.
	
	Noting that
	\begin{equation*}
	\begin{aligned}
	\mathbb{E}_S\left[\triangle^s\left(\mathcal{A}_\bold{w}(S),\mathcal{A}_\bold{v}(S)\right)-\triangle_S^s\left(\mathcal{A}_\bold{w}(S),\mathcal{A}_\bold{v}(S)\right)\right]&\leq\left\Vert\triangle^s\left(\mathcal{A}_\bold{w}(S),\mathcal{A}_\bold{v}(S)\right)-\triangle_S^s\left(\mathcal{A}_\bold{w}(S),\mathcal{A}_\bold{v}(S)\right)\right\Vert_2 \\
	&\leq200\sqrt{2}\left(1+\frac{L}{\rho}\right)G\gamma\lceil\log(n)\rceil+\frac{288G^2}{\rho\iota n}+\frac{16M_\ell}{n} \\
	&\quad+\iota\mathbb{E}_{S}\left[\triangle^s\left(\mathcal{A}_\bold{w}(S),\mathcal{A}_\bold{v}(S)\right)\right],
	\end{aligned}
	\end{equation*}
	where the last inequality holds because of the Cauchy-Schwartz inequality, which implies
	\begin{equation*}
	\mathbb{E}_S\left[\triangle^s\left(\mathcal{A}_\bold{w}(S),\mathcal{A}_\bold{v}(S)\right)\right]-\frac{1}{1-\iota}\mathbb{E}_S\left[\triangle_S^s\left(\mathcal{A}_\bold{w}(S),\mathcal{A}_\bold{v}(S)\right)\right]\leq\frac{1}{1-\iota}\left(200\sqrt{2}\left(1+\frac{L}{\rho}\right)G\gamma\lceil\log(n)\rceil+\frac{288G^2}{\rho\iota n}+\frac{16M_\ell}{n}\right).
	\end{equation*}
	
	Plugging the result back into (\ref{eq30}), then for $\tau\geq2$, we have
	\begin{equation*}
	\begin{aligned}
	&\left\Vert\triangle^s\left(\mathcal{A}_\bold{w}(S),\mathcal{A}_\bold{v}(S)\right)-\triangle_S^s\left(\mathcal{A}_\bold{w}(S),\mathcal{A}_\bold{v}(S)\right)\right\Vert_\tau \\
	&\leq\frac{\iota\tau}{1-\iota}\left(\mathbb{E}_S\left[\triangle_S^s\left(\mathcal{A}_\bold{w}(S),\mathcal{A}_\bold{v}(S)\right)\right]+200\sqrt{2}\left(1+\frac{L}{\rho}\right)G\gamma\lceil\log(n)\rceil+\frac{288G^2}{\rho\iota n}+\frac{16M_\ell}{n}\right) \\
	&\quad+100\sqrt{2}\left(1+\frac{L}{\rho}\right)G\tau\gamma\lceil\log(n)\rceil+\frac{144G^2\tau}{\rho\iota n}+\frac{8M_\ell\tau}{n}.
	\end{aligned}
	\end{equation*}
	
	Via Lemma \ref{lem3}, with probability at least $1-\zeta$, we have
	\begin{equation}\label{eq37}
	\begin{aligned}
	&\triangle^s\left(\mathcal{A}_\bold{w}(S),\mathcal{A}_\bold{v}(S)\right) \\
	&\leq\frac{100\sqrt{2}e(1+\iota)(1+L/\rho)G\gamma\log(n)}{1-\iota}\log\left(\frac{e}{\zeta}\right)+\frac{144e(1+\iota)G^2}{\rho\iota(1-\iota)n}\log\left(\frac{e}{\zeta}\right)+\frac{8e(1+\iota)M_\ell}{n(1-\iota)}\log\left(\frac{e}{\zeta}\right) \\
	&\quad+\frac{e\iota}{1-\iota}\log\left(\frac{e}{\zeta}\right)\mathbb{E}_S\left[\triangle_S^s\left(\mathcal{A}_\bold{w}(S),\mathcal{A}_\bold{v}(S)\right)\right]+\triangle_S^s\left(\mathcal{A}_\bold{w}(S),\mathcal{A}_\bold{v}(S)\right),
	\end{aligned}
	\end{equation}
	which completes the proof of part (d) and Theorem \ref{the3}.
	
\end{proof}

\subsection{A.4. Proof of Lemma \ref{lem9}}
\begin{lemma}
	If Assumptions \ref{a1} and \ref{a2} hold.
	Taking $\sigma$ given in Theorem \ref{the1}, and $\eta_t=\frac{1}{\rho t}$, then with probability at least $1-\zeta$ for $\zeta\in(\exp(-\frac{p}{8}),1)$, the strong primal dual empirical risk of the output of DP-GDA: $\mathcal{A}(S)=(\bar{\bold{w}}_T,\bar{\bold{v}}_T)$ satisfies
	\begin{equation*}
	\begin{aligned}
	&\triangle_S^s\left(\bar{\bold{w}}_T,\bar{\bold{v}}_T\right) \\
	&\leq \frac{G^2\log(eT)}{\rho T}+\frac{cG\left(g_\bold{w}+g_\bold{v}\right)\sqrt{Tp\log(1/\delta)}}{n\epsilon}p_\zeta+cG^2\log(eT)\left(\frac{p\log(1/\delta)}{\rho n^2\epsilon^2}p_\zeta^2+\frac{2\sqrt{p\log(1/\delta)}}{\rho\sqrt{T}n\epsilon}p_\zeta\right),
	\end{aligned}
	\end{equation*}
	for some constant $c$, where $p_\zeta,g_\bold{w},g_\bold{v}$ are defined as in Theorem \ref{the2}.
\end{lemma}

\begin{proof}
	In Appendix A.2, we have analyzed the strong PD empirical risk in (\ref{eq33}).
	So, via the connection between the argument stability and the strong PD empirical risk (see (\ref{eq34})), along with plugging the random noise $\bold{w}$ and $\bold{v}$ into it (see (\ref{eq35})), we have
	\begin{equation*}
	\triangle_S^s\left(\bar{\bold{w}}_{T},\bar{\bold{v}}_{T}\right)\leq\frac{G^2\log(eT)}{\rho T}+\frac{\sigma^2p\log(eT)}{\rho T}p_\zeta^2+\frac{2G\sigma\sqrt{p}\log(eT)}{\rho T}p_\zeta+\left(g_\bold{w}+g_\bold{v}\right)\sigma\sqrt{p}p_\zeta,
	\end{equation*}
	where $p_\zeta=1+\big(\frac{8\log(2T/\zeta)}{p}\big)^{1/4}$, $g_\bold{w}=\left\Vert\bar{\bold{w}}^*-\bar{\bold{w}}_{T}\right\Vert_2$, and $g_\bold{v}=\left\Vert\bar{\bold{v}}^*-\bar{\bold{v}}_{T}\right\Vert_2$ for $\bar{\bold{w}}^*=\arg\min_{\bold{w}\in\mathcal{W}}L_S(\bold{w},\bar{\bold{v}}_T)$, and $\bar{\bold{v}}^*=\arg\max_{\bold{v}\in\mathcal{V}}L_S(\bar{\bold{w}}_T,\bold{v})$.
	
	Taking $\sigma$ given in Theorem \ref{the1}, then the proof completes.
\end{proof}

\subsection{A.5. Proof of Corollary \ref{cor1}}
\begin{corollary}
	With argument stability parameter $\gamma$,
	
	(a) Under the condition given in Theorem \ref{the3} part (d), the strong PD generalization error satisfies
	\begin{equation*}
	\begin{aligned}
	&\triangle^s\left(\mathcal{A}_\bold{w}(S),\mathcal{A}_\bold{v}(S)\right)-\triangle_S^s\left(\mathcal{A}_\bold{w}(S),\mathcal{A}_\bold{v}(S)\right) \\
	&\leq\frac{100\sqrt{2}e(1+\iota)(1+L/\rho)G\gamma\log(n)}{1-\iota}\log\left(\frac{e}{\zeta}\right)+\frac{144e(1+\iota)G^2}{\rho\iota(1-\iota)n}\log\left(\frac{e}{\zeta}\right)+\frac{8e(1+\iota)M_\ell}{n(1-\iota)}\log\left(\frac{e}{\zeta}\right) \\
	&\quad+\frac{e\iota}{1-\iota}\log\left(\frac{e}{\zeta}\right)\mathbb{E}_S\left[\triangle_S^s\left(\mathcal{A}_\bold{w}(S),\mathcal{A}_\bold{v}(S)\right)\right].
	\end{aligned}
	\end{equation*}
	
	(b) Under the condition given in Theorem \ref{the3} part (d), the weak primal dual population risk satisfies
	\begin{equation*}
	\begin{aligned}
	&\triangle^w\left(\mathcal{A}_\bold{w}(S),\mathcal{A}_\bold{v}(S)\right) \\
	&\leq\frac{100\sqrt{2}e(1+\iota)(1+L/\rho)G\gamma\log(n)}{1-\iota}\log\left(\frac{e}{\zeta}\right)+\frac{144e(1+\iota)G^2}{\rho\iota(1-\iota)n}\log\left(\frac{e}{\zeta}\right)+\frac{8e(1+\iota)M_\ell}{n(1-\iota)}\log\left(\frac{e}{\zeta}\right) \\
	&\quad+\left(\frac{e\iota}{1-\iota}\log\left(\frac{e}{\zeta}\right)+1\right)\triangle_S^s\left(\mathcal{A}_\bold{w}(S),\mathcal{A}_\bold{v}(S)\right).
	\end{aligned}
	\end{equation*}
	
	(c) Under the condition given in Theorem \ref{the3} part (d), the weak PD generalization error satisfies
	\begin{equation*}
	\begin{aligned}
	&\triangle^w\left(\mathcal{A}_\bold{w}(S),\mathcal{A}_\bold{v}(S)\right)-\triangle_S^w\left(\mathcal{A}_\bold{w}(S),\mathcal{A}_\bold{v}(S)\right) \\
	&\leq\frac{100\sqrt{2}e(1+\iota)(1+L/\rho)G\gamma\log(n)}{1-\iota}\log\left(\frac{e}{\zeta}\right)+\frac{144e(1+\iota)G^2}{\rho\iota(1-\iota)n}\log\left(\frac{e}{\zeta}\right)+\frac{8e(1+\iota)M_\ell}{n(1-\iota)}\log\left(\frac{e}{\zeta}\right) \\
	&\quad+\left(\frac{e\iota}{1-\iota}\log\left(\frac{e}{\zeta}\right)+2\right)\triangle_S^s\left(\mathcal{A}_\bold{w}(S),\mathcal{A}_\bold{v}(S)\right).
	\end{aligned}
	\end{equation*}
\end{corollary}

\begin{proof}
	Part (a) can be directly derived from (\ref{eq37}), by removing $\triangle_S^s\left(\mathcal{A}_\bold{w}(S),\mathcal{A}_\bold{v}(S)\right)$ from the right side of the inequality to the left side.
	
	Part (b) holds because $\triangle^{w}(\bold{w},\bold{v})\leq\mathbb{E}\left[\triangle^{s}(\bold{w},\bold{v})\right]$, as discussed in Remark \ref{rem1}.
	
	For Part (c), as discussed in Remark \ref{rem1}, we first have $\triangle^{w}(\bold{w},\bold{v})\leq\mathbb{E}\left[\triangle^{s}(\bold{w},\bold{v})\right]$ and $\triangle_S^{w}(\bold{w},\bold{v})\leq\mathbb{E}\left[\triangle_S^{s}(\bold{w},\bold{v})\right]$, then by Jensen's inequality, we have $\triangle^{w}(\bold{w},\bold{v})-\triangle_S^{w}(\bold{w},\bold{v})\leq\triangle^{w}(\bold{w},\bold{v})+|\triangle_S^{w}(\bold{w},\bold{v})|\leq\mathbb{E}\left[\triangle^{s}(\bold{w},\bold{v})\right]+|\mathbb{E}\left[\triangle_S^{s}(\bold{w},\bold{v})\right]|$.
	
	As a result, via (\ref{eq37}), we have
	\begin{equation*}
	\begin{aligned}
	&\triangle^{w}(\mathcal{A}_\bold{w}(S),\mathcal{A}_\bold{v}(S))-\triangle_S^{w}(\mathcal{A}_\bold{w}(S),\mathcal{A}_\bold{v}(S)) \\
	&\leq\frac{100\sqrt{2}e(1+\iota)(1+L/\rho)G\gamma\log(n)}{1-\iota}\log\left(\frac{e}{\zeta}\right)+\frac{144e(1+\iota)G^2}{\rho\iota(1-\iota)n}\log\left(\frac{e}{\zeta}\right)+\frac{8e(1+\iota)M_\ell}{n(1-\iota)}\log\left(\frac{e}{\zeta}\right) \\
	&\quad+\left(\frac{e\iota}{1-\iota}\log\left(\frac{e}{\zeta}\right)+2\right)\triangle_S^s\left(\mathcal{A}_\bold{w}(S),\mathcal{A}_\bold{v}(S)\right),
	\end{aligned}
	\end{equation*}
	where we couple $\triangle_S^s\left(\mathcal{A}_\bold{w}(S),\mathcal{A}_\bold{v}(S)\right)$ together with its expectation and omit the absolute value of $\mathbb{E}\left[\triangle_S^{s}(\bold{w},\bold{v})\right]$ because it is the upper bound of $\triangle_S^s\left(\mathcal{A}_\bold{w}(S),\mathcal{A}_\bold{v}(S)\right)$ who matters.
	
	The proof completes.
	
\end{proof}

\end{appendix}

\end{document}